\documentclass[reqno]{amsart}
\pdfoutput=1 
\usepackage[margin=3.5cm]{geometry}
\newcommand{\forarxiv}[1]{#1}
\newcommand{\foraistats}[1]{}

\RequirePackage{times}
\RequirePackage{fancyhdr}
\RequirePackage{xcolor}
\RequirePackage{algorithm}
\RequirePackage{algorithmic}
\RequirePackage{natbib}
\RequirePackage{eso-pic}
\RequirePackage{forloop}
\RequirePackage{url}

\usepackage{microtype}
\usepackage{graphicx}
\usepackage{subfigure}
\usepackage{hyperref}

\usepackage{amsmath}
\usepackage{amssymb}
\usepackage{mathtools}
\usepackage{amsthm}

\usepackage{amsfonts}
\mathtoolsset{showonlyrefs}
\usepackage{enumitem}
\usepackage{bbm}
\usepackage{upgreek}
\usepackage{stmaryrd}
\usepackage{xspace}
\usepackage{caption}
\usepackage{amsaddr}
\usepackage{footmisc}

\newcommand{\thp}{\theta}
\newcommand{\thprop}{\lambda}
\newcommand{\Norm}{N}

\newcommand{\nset}{\mathbb{N}}
\newcommand{\nsetpos}{\mathbb{N}_{>0}}
\newcommand{\rset}{\mathbb{R}}
\newcommand{\rsetnonneg}{\mathbb{R}_+}
\newcommand{\rsetpos}{\mathbb{R}_+^*}
\newcommand{\1}[1]{\mathbbm{1}_{#1}}

\newcommand{\E}{\mathbb{E}}
\newcommand{\prob}{\mathbb{P}}
\newcommand{\repfunc}[1]{f_{#1}}
\newcommand{\grad}[1]{H_{#1}}
\newcommand{\htil}[1]{\tilde{H}_{#1}}
\newcommand{\mf}[1]{h(#1)}
\newcommand{\set}[1]{\mathsf{#1}}
\newcommand{\alg}[1]{\mathcal{#1}}
\newcommand{\meas}[1]{\mathsf{M}(#1)}
\newcommand{\probmeas}[1]{\mathsf{M}_1(#1)}

\newcommand{\parspace}{\mathsf{\Theta}}
\newcommand{\partlaw}[1]{Q^N_{#1, \theta}}
\newcommand{\propspace}{\mathsf{\Lambda}}
\newcommand{\catdist}{\mathsf{cat}}

\newcommand{\tensprod}{\varotimes}
\newcommand{\eqdef}{\vcentcolon=}
\newcommand{\ie}{\emph{i.e.}\xspace}
\newcommand{\eg}{\emph{e.g.}\xspace}
\newcommand{\elbo}{\mathcal{L}}
\newcommand{\vi}{\texttt{VI}}
\newcommand{\iwae}{{\tiny{\texttt{IWAE}}}}
\newcommand{\ADAM}{\texttt{ADAM}}
\newcommand{\IWAE}{\texttt{IWAE}}
\newcommand{\smc}{\texttt{SMC}}
\newcommand{\VSMC}{\texttt{VSMC}}
\newcommand{\OVSMC}{\texttt{OVSMC}}
\newcommand{\OVF}{\texttt{OVF}}
\newcommand{\gradapprox}{\mathcal{G}}

\newcommand{\norm}[1]{\lVert #1 \rVert}
\newcommand{\abs}[1]{\left\lvert #1 \right\rvert}

\newcommand{\noteJO}[1]{#1}

\newcommand{\z}{Z}
\newcommand{\epart}[2]{\xi_{#1}^{#2}}
\newcommand{\epartil}[2]{\tilde{\xi}_{#1}^{#2}}
\newcommand{\epartb}[2]{x_{#1}^{#2}}
\newcommand{\epartilb}[2]{\tilde{x}_{#1}^{#2}}
\newcommand{\I}[2]{I_{#1}^{#2}}
\newcommand{\auxrv}{\varepsilon}
\newcommand{\auxrvb}{v}
\newcommand{\wgt}[2]{\omega_{#1}^{#2}}
\newcommand{\wgtsum}[1]{\Omega_{#1}}
\newcommand{\wgtfunc}[2]{\omega_{#1}^{#2}}
\newcommand{\wgtfuncb}[2]{w_{#1}^{#2}}

\newcommand{\tz}[1]{T_{#1}}
\newcommand{\tzc}[1]{\tilde{T}_{#1}}
\newcommand{\hidker}[1]{M_{#1}}
\newcommand{\hidkertrue}{\bar{M}}
\newcommand{\hiddens}[1]{m_{#1}}
\newcommand{\hiddenstrue}{\bar{m}}
\newcommand{\emker}[1]{G_{#1}}
\newcommand{\emkertrue}{\bar{G}}
\newcommand{\emdens}[1]{g_{#1}}
\newcommand{\emdenstrue}{\bar{g}}

\newcommand{\propdens}[1]{r_{#1}}
\newcommand{\propker}[1]{R_{#1}}
\newcommand{\kernel}[1]{#1}

\newcommand{\indmeas}{\nu}
\newcommand{\statmeas}[1]{\tau_{#1}}

\newcommand{\xinit}{\chi}
\newcommand{\xinitdens}[1]{m_{#1}}

\newcommand{\rhoq}{\varrho}
\newcommand{\cq}{\varsigma}
\newcommand{\kq}{\kappa}
\newcommand{\kqt}{\tilde{\kappa}}
\newcommand{\eq}{\epsilon}

\newcommand{\aq}{\tilde{\alpha}}
\newcommand{\ch}{c_h}

\theoremstyle{plain}
\newtheorem{theorem}{Theorem}[section]
\newtheorem{proposition}[theorem]{Proposition}
\newtheorem{lemma}[theorem]{Lemma}
\newtheorem{corollary}[theorem]{Corollary}
\theoremstyle{definition}

\newtheorem{assumption}[theorem]{Assumption}
\theoremstyle{remark}
\newtheorem{remark}[theorem]{Remark}

\title[Online Variational Sequential Monte Carlo]{\Large Online Variational Sequential Monte Carlo}
\author[A. Mastrototaro \and J. Olsson]{Alessandro Mastrototaro \and Jimmy Olsson}
\address{\small Department of Mathematics, KTH Royal Institute of Technology, Stockholm, Sweden}
\email{alemas@kth.se, jimmyol@kth.se}

\begin{document}

\begin{abstract}
   %Following some recent advances on the connection between \emph{variational inference} and \emph{sequential Monte Carlo} methods, or \emph{particle filters}, we present an extension that allows for online learning in \emph{state-space models} (SSM). Previous approaches either dealt with batch processing of the data or with some time-consuming  and only locally optimal approach. Instead, we distribute the approximate gradient of some surrogate \emph{evidence lower bound} (ELBO) in time, allowing for sequential learning in the presence of streams of data. The result is an algorithm that may estimate the parameters of the model and, at the same time, learn proposal distributions that may improve the estimated posterior distribution of the latent states. We demonstrate the performances of our methodology in some well known SSMs, showing in addition that the methodology may also be used to accelerate its corresponding version if applied to non-streaming partial observations.
%
%\noteJO{Jimmy's version:}

Being the most classical generative model for serial data, state-space models (SSM) are fundamental in AI and statistical machine learning. In SSM, any form of parameter learning or latent state inference typically involves the computation of complex latent-state posteriors. In this work, we build upon the variational sequential Monte Carlo ({\VSMC}) method, which provides computationally efficient and accurate model parameter estimation and Bayesian latent-state inference by combining particle methods and variational inference. While standard {\VSMC} operates in the offline mode, by re-processing repeatedly a given batch of data, we distribute the approximation of the gradient of the {\VSMC} surrogate ELBO in time using stochastic approximation, allowing for online learning in the presence of streams of data. This results in an algorithm, online {\VSMC}, that is capable of performing efficiently, entirely on-the-fly, both parameter estimation and particle proposal adaptation. In addition, we provide rigorous theoretical results describing the algorithm's convergence properties as the number of data tends to infinity as well as numerical illustrations of its excellent convergence properties and usefulness also in batch-processing settings.

\end{abstract}
\begingroup
\let\MakeUppercase\relax % this disables uppercasing authors
\maketitle
\endgroup

\section{Introduction}\label{sec:intro}
%\noteJO{New version:}
\forarxiv{\let\thefootnote\relax\footnotetext{Published as a conference paper at ICML 2024.}}

Being the most classical structured probabilistic generative model for serial data, \emph{state-space models} (SSM), also known as general state-space \emph{hidden Markov models}, are fundamental and ubiquitous in AI and statistical machine learning \citep{cappe:moulines:ryden:2005,bishop:2016}. In SSM, any form of parameter learning or state inference typically involves the computation of complex joint posterior distributions of latent-state variables---the so-called \emph{joint-smoothing distributions}---given records of observations, which is a delicate problem whose analytical solution is intractable outside the limited cases of linear Gaussian models or models with finite state space. In a recent line of research \citep{le:2018,maddison:2017,naesseth:2018}, this problem is addressed by combining \emph{variational inference} \citep{blei:2017,kingma:welling:2014} and \emph{sequential Monte Carlo} (SMC) \emph{methods} \citep{gordon:salmond:smith:1993,doucet:defreitas:gordon:2001,chopin:papaspiliopoulos:2020} in order to design flexible families of variational joint-state distributions in the form of particle-trajectory laws. By optimizing the Kullback--Leibler  divergence from (KLD) the law of the particle trajectories to the joint-smoothing distribution, this \emph{variational SMC} ({\VSMC}) approach is killing two birds with one stone by learning not only the unknown model parameters but also, in parallel, an optimal particle proposal kernel, the latter being a problem that has received a lot of attention in the literature \citep{doucet:godsill:andrieu:2000,cornebise:moulines:olsson:2008,gu:2015}. The procedure can be viewed as a non-trivial extension of the \emph{importance-weighted auto-encoder} ({\IWAE}) proposed by \citet{burda:2016}, where standard self-normalized importance sampling has been replaced by sequential importance sampling with systematic resampling in order to obtain a tighter \emph{evidence lower bound} (ELBO). The objective of {\VSMC} is a similar ELBO whose gradient is approximated by the expectation, under the law of the random numbers generated by the SMC algorithm, of the gradient of the logarithm of the standard particle-based likelihood estimator, the latter being unbiased \citep{delmoral:2004,chopin:papaspiliopoulos:2020}, is obtained as a by-product of the particle approximation of the marginal-state posterior---or \emph{filter distribution}---flow;  thus, whereas traditional particle-based inference in SSM typically relies on particle approximation of the joint-smoothing distributions (with aim of approximating, \emph{e.g.}, the intermediate quantity of the expectation--maximization, EM, algorithm or the score function directly via the Fisher identity), which is cumbersome due to the so-called particle-path degeneracy phenomenon \citep[Section~8.3]{kitagawa:1996,cappe:moulines:ryden:2005}, {\VSMC} allows, using the reparameterization trick, the model parameters as well as an optimal particle proposal kernel to be simultaneously learned by processing repeatedly the given data batch using a standard particle filter. 

In its basic form, {\VSMC} is an offline inference technique, in the sense that it requires the full data batch to be processed between every update of the model and variational parameters. Still, in a wide range of AI and machine-learning contexts, observed data become available sequentially through a data stream, requiring learning to be performed in an online fashion. The increasing interest in online machine-learning technology also stems from the need of processing data batches of ever-increasing sizes. Thus, in this work we propose an online, stochastic approximation-based version of {\VSMC}, \emph{online variational SMC} ({\OVSMC}), which can be used for simultaneous online model learning and proposal adaptation. On the contrary to traditional approaches to particle-based online parameter learning in SSM such as particle-based recursive maximum-likelihood \citep[RML,][]{legland:mevel:1997,delmoral:doucet:singh:2015} or online EM \citep{mongillo:deneve:2008,cappe:2009}, which rely on particle-based online approximation of the tangent filter (filter derivative) and the EM-intermediate quantity, respectively, our approach does not involve any particle smoothing, which, in order to avoid the problem of collapsing particle ancestral genealogies, typically calls for backward-sampling techniques which can be computationally very costly. 
In a recent work, \citet{campbell:2021} used variational approximations of the backward kernels for online state and parameter learning. Although this, interestingly, provides a non-particle-based methodology, it is very computationally intensive (see Section~\ref{sec:experiments}).
In addition, our method allows for effective online adaptation of the particle proposal kernel in a way that differs from typical approaches in the literature; indeed, whereas traditional approaches aim to optimize the proposal time step by time step on the basis of local criteria \citep[see, \eg,][]{cornebise:moulines:olsson:2008,cornebise:moulines:olsson:2014,zhao:2021}, our approach is based on a global KLD-based criterion (described in detail in Section~\ref{sec:method}) which allows the particle cloud to be effectively guided toward state-space regions of high local likelihood, without running the risk of over-adapting locally the proposal to current (or temporarily neighboring) observations. 

Although our proposed method has similarities with the \emph{streaming variational Monte Carlo} methodology proposed by \citet{zhao:2021}, it is essentially different from the same. An important difference is that the aforementioned work focuses on local optimization of model and variational parameters, such that these are assumed to vary with time and are therefore optimized time step by time step, while {\OVSMC} estimate global, amortized parameters using stochastic approximation. 
%and provide a rigorous theoretical analysis of the obtained estimates as the amount of data grows. 
%An approach that is more oriented towards online inference is presented by \cite{zhao:2021}; still, the approach proposed in the latter work differs from {\OVSMC} is that their work is focused on local optimization of model and proposal parameters, \ie they assume that these may be different at each time step, while our goal is to determine global parameterizations. In addition, we furnish our algorithm with rigorous convergence results. 
In addition, in order to show that our approach is statistically well founded, we provide rigorous theoretical results describing its convergence (Theorem~\ref{thm:main}). Under certain strong mixing assumptions, the time-normalized gradient guiding the learning of batch {\VSMC} can, using Birkhoff's ergodic theorem, be shown to converge as the observation batch size increases towards infinity to a deterministic function depending on the parameter as well as the particle sample size; we show that as the number of observations tends to infinity, {\OVSMC} is solving the same problem as this ideal, `asymptotic' {\VSMC}, in the sense that the mean field targeted by {\OVSMC} coincides with the time-normalized asymptotic gradient of {\VSMC}. 

Finally, we illustrate {\OVSMC} numerically on a number of classical SSM and more complex generative models, for which the method exhibits fast parameter learning and efficient adaptation of the particle proposal kernel. In the same numerical study, we also show that {\OVSMC} is a strong challenger of {\VSMC} on batch problems. 
%, in the sense that the former requires fewer sweeps over the given data batch than the latter in order to reach convergence. 

The paper is structured as follows. In Section~\ref{sec:background}, we review particle filtering in the context of SSM, and their relation with variational inference in some recent works. In Section~\ref{sec:method}, we present our methodology for online learning of proposal distributions and parameters and Section~\ref{sec:theory} and Section~\ref{sec:experiments} provide our theoretical results and numerical experiments, respectively. 

\section{Background}\label{sec:background}
An SSM is a bivariate, time-homogeneous Markov chain $(X_t,Y_t)_{t \in \nset}$ evolving on some general measurable product space $(\set{X} \times \set{Y}, \alg{X} \tensprod \alg{Y})$. In most applications, $(\set{X}, \alg{X})$ and $(\set{Y}, \alg{Y})$ are Euclidean and furnished with the corresponding Borel $\sigma$-fields. More specifically, the marginal process $(X_t)_{t \in \nset}$, referred to as the \emph{state process}, is itself assumed to be a time-homogeneous Markov with transition density $\hiddens{\thp}(x_{t + 1} \mid x_t)$ and initial-state density $\xinitdens{0}(x_0)$ (with respect to the same dominating measure $dx$, typically the Lebesgue measure) on $\set{X}$. The state process is latent but partially observed through the \emph{observation process} $(Y_t)_{t \in \nset}$, whose values are assumed conditionally independent given the state process such that the conditional distribution of each $Y_t$ depends on the corresponding $X_t$ only, and we denote by $\emdens{\thp}(y_t \mid x_t)$ the density (with respect to some dominating measure) on $\set{Y}$ of the latter. Using this notation, the joint density of a given vector $X_{0:t} = (X_0, \ldots, X_t)$ (this is our generic notation for vectors) of states and corresponding observations $Y_{0:t}$ is given by   
\foraistats{ \begin{multline} \label{eq:def:joint:density}
 p_\thp(x_{0:t}, y_{0:t}) = \xinitdens{0}(x_0)\emdens{\thp}(y_0\mid x_0)
	\\ \times \prod_{s=0}^{t-1}\hiddens{\thp}(x_{s+1}\mid x_s)\emdens{\thp}(y_{s+1}\mid x_{s+1}). 
 \end{multline}}
\forarxiv{ \begin{equation} \label{eq:def:joint:density}
		p_\thp(x_{0:t}, y_{0:t}) = \xinitdens{0}(x_0)\emdens{\thp}(y_0\mid x_0)\prod_{s=0}^{t-1}\hiddens{\thp}(x_{s+1}\mid x_s)\emdens{\thp}(y_{s+1}\mid x_{s+1}). 
\end{equation}}

%We denote by $m_\thp^0(x_0)$ the density of the initial distribution of the state process with respect to some reference measure, then the SSM is entirely defined by that and by the Markov transition kernel between the hidden states and the emission distribution of the observations conditional to the corresponding latent state. These kernels have conditional densities $\hiddens{\thp}(x_{t+1} \mid x_t)$ and $\emdens{\thp}(y_t \mid x_t)$, $t \in \nset$, with respect to some dominating measures, which we may assume to be Lebesgue ones. \noteJO{With respect to some dominating measure. Rewrite.} \noteAM{Done.} 

The model dynamics is governed by some parameter vector $\thp$ belonging to some parameter space $\parspace$. 
%, and we assume that the latent and observed data $(X_t,Y_t)_{t\in\nset}$ is generated according to some true unknown value in $\parspace$. \noteJO{Shorten.} \noteAM{Done.} 
In using these models, the focus is generally on inferring the hidden states given data, typically by determining the joint-smoothing distributions $p_\thp(x_{0:t} \mid y_{0:t}) = p_\thp(x_{0:t}, y_{0:t}) / p_\thp(y_{0:t})$ or their marginals,  such as the filter distributions $p_\thp(x_{t}\mid y_{0:t})$. Here the joint density $p_\thp(x_{0:t}, y_{0:t})$ is given by \eqref{eq:def:joint:density}, whereas the likelihood $p_\thp(y_{0:t})$ of the observations is the marginal of \eqref{eq:def:joint:density} with respect to $y_{0:t}$. 
%write it as
%\begin{multline}
%	p_\thp(x_{0:t}\mid y_{0:t})=(p_\thp(y_{0:t}))^{-1}m_\thp^0(x_0)\emdens{\thp}(y_0\mid x_0)
%	\\\times\prod_{s=0}^{t-1}\hiddens{\thp}(x_{s+1}\mid x_s)\emdens{\thp}(y_{s+1}\mid x_{s+1}),\quad t\ge 0,
%\end{multline}
%where $p_\thp(y_{0:t})$ indicates the likelihood of the observations. 
As the calculation of the likelihood requires the computation of complex integrals in high dimensions, the joint-smoothing and filter distributions are generally---except in the cases where the model is linear Gaussian or the state space $\set{X}$ is a finite set---intractable. The calculation of the joint-smoothing distributions is of critical importance also when the parameter $\thp$ is unknown and must be estimated using maximum-likelihood or Bayesian methods \citep[see][and the references therein]{cappe:moulines:ryden:2005,kantas:doucet:singh:chopin:2015}. In the framework of SSMs we may distinguish between \emph{batch} methods, where parameter and state inference is carried through given a fixed record of observations, and \emph{online} methods, where inference is carried through in real time as new data becomes available through a data stream. In the online mode, SMC algorithms are particularly well suited for state inference; furthermore, these methods also provide a basis for online parameter estimation \citep{poyiadjis:doucet:singh:2011,delmoral:doucet:singh:2015,olsson:westerborn:2018}. We next review briefly SMC and the principles of variational inference. 

\subsection{Sequential Monte Carlo Methods}
\label{subsec:smc}
In the context of SSM, SMC methods approximate the smoothing distribution flow $(p_\thp(x_{0:t}\mid y_{0:t}))_{t \in \nset}$ by forming iteratively a sequence $(\epart{0:t}{i}, \wgt{t}{i})_{i=1}^N$, $t \in \nset$, of random samples of particles (the $\epart{0:t}{i}$'s) with associated weights (the $\wgt{t}{i}$'s). For $t \in \nset$, let $\set{X}^t \eqdef \set{X} \times \cdots \times \set{X}$ ($t$ times); 
%Let now $(\set{X}^1,\alg{X}^{\tensprod 1})\eqdef(\set{X},\alg{X})$ and define recursively, for $t=1,2,\dots$, the product spaces $(\set{X}^{t+1},\alg{X}^{\tensprod(t+1)})\eqdef(\set{X}\times \set{X}^{t},\alg{X}\tensprod\alg{X}^{\tensprod t})$,
 then for any real-valued measurable function $h_t$ on $\set{X}^{t+1}$ that is  integrable with respect to $p_\thp(x_{0:t}\mid y_{0:t})$, it holds, letting $\wgtsum{t}\eqdef\sum_{i=1}^{N}\wgt{t}{i}$,
\begin{equation}
	\sum_{i=1}^{N}\frac{\wgt{t}{i}}{\wgtsum{t}}h_t(\epart{0:t}{i})\backsimeq  \int h_t(x_{0:t})p_\thp(x_{0:t}\mid y_{0:t})\,dx_{0:t}.
\end{equation}
See \citet{delmoral:2004} for a comprehensive treatment of the theory of SMC. The SMC procedure consists of two core operations performed alternately: a \emph{selection step}, which resamples the particles with replacement according to their weights, and a \emph{mutation step}, which propagates randomly the selected particles to new locations. After importance sampling-based initialization of $(\epart{0}{i}, \wgt{0}{i})_{i=1}^N$, the random sample is updated according to Algorithm~\ref{algo:pf}, in which mutation (Line~4) is executed using some generic proposal Markov transition kernel, possibly depending on $y_{t+1}$, with density $\propdens{\thprop}(x_{t+1} \mid x_t, y_{t+1})$ parameterized by some vector $\thprop$ belonging to some parameter space $\propspace$. Line~3 corresponds to the selection step, where indices $(\I{t+1}{i})_{i = 1}^N$ guiding the resampling are drawn from the categorical distribution on $\{1,\dots,N\}$ induced by the particle weights. 

%\begin{algorithm}[htb]
%	\caption{Particle filter}\label{algo:pf}
%	\begin{algorithmic}[1]
%		\Statex \textbf{Input}: $(\epart{0:t}{i},\wgt{t}{i})_{i=1}^N, y_{t+1}$.
%		\For{$i=1: N$}
%		\State draw $\I{t+1}{i}\sim \catdist((\wgt{t}{\ell})_{\ell=1}^N)$;
%		%\Comment{{\tiny Categorical distribution on $\{1,\dots,N\}$ with probabilities proportional to the weights;}}
%		\label{line:res}
%		\State draw $\epart{t+1}{i}\sim\propdens{\thprop}(\cdot\mid \epart{t}{\I{t+1}{i}},y_{t+1})$;\label{line:pfprop}
%		\State set $\epart{0:t+1}{i}\gets(\epart{0:t}{\I{t+1}{i}},\epart{t+1}{i})$;\label{line:attach_part}
%		\State set $\wgt{t+1}{i}\gets  \dfrac{\hiddens{\thp}(\epart{t+1}{i}\mid \epart{t}{\I{t+1}{i}})\emdens{\thp}(y_{t+1}\mid \epart{t+1}{i})}{\propdens{\thprop}(\epart{t+1}{i}\mid \epart{t}{\I{t+1}{i}},y_{t+1})}$;
%		\EndFor
%		\State\Return $(\epart{0:t+1}{i},\wgt{t+1}{i})_{i=1}^N$.
%	\end{algorithmic}
%\end{algorithm}
\begin{algorithm}[htb]
	\caption{Particle filter}\label{algo:pf}
	\begin{algorithmic}[1]
		\STATE {\bfseries Input:} $(\epart{0:t}{i},\wgt{t}{i})_{i=1}^N, y_{t+1}$.
		\FOR{$i=1,\dots,N$}
		\STATE draw $\I{t+1}{i}\sim \catdist((\wgt{t}{\ell})_{\ell=1}^N)$;\label{line:res}
		\STATE draw $\epart{t+1}{i}\sim\propdens{\thprop}(\cdot\mid \epart{t}{\I{t+1}{i}},y_{t+1})$;\label{line:pfprop}
		\STATE set $\epart{0:t+1}{i}\gets(\epart{0:t}{\I{t+1}{i}},\epart{t+1}{i})$;\label{line:attach_part}
		\STATE set $\wgt{t+1}{i}\gets  \dfrac{\hiddens{\thp}(\epart{t+1}{i}\mid \epart{t}{\I{t+1}{i}})\emdens{\thp}(y_{t+1}\mid \epart{t+1}{i})}{\propdens{\thprop}(\epart{t+1}{i}\mid \epart{t}{\I{t+1}{i}},y_{t+1})}$;
		\ENDFOR
		\STATE {\bfseries return} $(\epart{0:t+1}{i},\wgt{t+1}{i})_{i=1}^N$.
	\end{algorithmic}
\end{algorithm}

Determining a good proposal distribution $\propdens{\thprop}$ is crucial for the performance of the particle filter  \citep[see, \eg,][]{cornebise:moulines:olsson:2008,cornebise:moulines:olsson:2014,gu:2015,zhao:2021}, and the particles should be guided towards state-space regions of non-vanishing likelihood %density of observations 
in order to avoid computational waste. For instance, if the latent process is diffuse %has high variance 
while the observations are highly informative, letting naively, as in the so-called \emph{bootstrap particle filter} \citep{gordon:salmond:smith:1993}, $\propdens{\thprop} \equiv \hiddens{\thp}$ may result in many particles being assigned a negligible weight and thus significant sample depletion. A more appealing, data-driven option is to use the 
 %As we mentioned, the new observation $y_{t+1}$ may be taken into account, \eg, when using the 
\emph{locally optimal proposal} satisfying $\propdens{\thprop}(x_{t + 1} \mid x_t, y_{t + 1})\propto\hiddens{\thp}(x_{t+1}\mid x_t)\emdens{\thp}(y_{t + 1} \mid x_{t + 1})$ (and minimizing, \eg, the Kullback--Leibler divergence \cite{cornebise:moulines:olsson:2008}); however, this proposal is available in a closed form only in a few cases \citep[Section~7.2.2.2]{doucet:godsill:andrieu:2000,cappe:moulines:ryden:2005}. In Section~\ref{sec:method} we will present a technique that allows to learn simultaneously, in an online fashion, both unknown model parameters and an efficient proposal. The proposed method is based on variational inference, which is briefly reviewed in the next section. 

\subsection{Variational Inference}\label{subsec:varinf}
Let $T \in \nset$ be a fixed time horizon. To approximate $p_\thp(x_{0:T}\mid y_{0:T})$ by a variational-inference procedure, a family $\{ q_\thprop(x_{0:T}\mid y_{0:T}) : \thprop \in \propspace \}$ of \emph{variational distributions} is designed, whereby the ELBO 
% Taking a variation inference (VI) approach to the approximation of $p_\thp(x_{0:T}\mid y_{0:T})$ involves designing a family $\{ q_\thprop(x_{0:T}\mid y_{0:T}) : \thprop \in \propspace \}$ of \emph{variational distributions}  and maximizing, 
%In standard variational inference (VI) \cite{jordan:ghahramani:jaakkola:saul:1999,wainwright:jordan:2008,kingma:welling:2014,blei:2017}, here applied to SSMs, we introduce a tractable \emph{variational distribution} $q_\thprop(x_{0:T}\mid y_{0:T})$ and optimize its parameters such that it approximates $p_\thp(x_{0:T}\mid y_{0:T})$. Having defined such a family of distribution, we maximize, 
%with respect to $(\thprop,\thp)$, the ELBO %$\elbo^\vi$ defined as
\foraistats{\begin{multline}
	\elbo^\vi(\thprop, \thp)\eqdef\E_{q_\thprop}\left[\log\left(\frac{p_\thp(X_{0:T},y_{0:T})}{q_\thprop(X_{0:T}\mid y_{0:T})}\right)\right]
	\\\le\log \E_{q_\thprop}\left[ \frac{p_\thp(X_{0:T},y_{0:T})}{q_\thprop(X_{0:T}\mid y_{0:T})} \right]= \log p_\thp(y_{0:T})
\end{multline}}\forarxiv{\begin{equation}
		\elbo^\vi(\thprop, \thp)\eqdef\E_{q_\thprop}\left[\log\left(\frac{p_\thp(X_{0:T},y_{0:T})}{q_\thprop(X_{0:T}\mid y_{0:T})}\right)\right]
		\le\log \E_{q_\thprop}\left[ \frac{p_\thp(X_{0:T},y_{0:T})}{q_\thprop(X_{0:T}\mid y_{0:T})} \right]= \log p_\thp(y_{0:T})
\end{equation}}
is maximized with respect to $(\thprop,\thp)$. Here the bound follows from Jensen's inequality. We note that the equality is satisfied in the ideal case  $q_\thprop(x_{0:T}\mid y_{0:T}) = p_\thp(x_{0:T}\mid y_{0:T})$. The optimization is performed by a combination of Monte Carlo sampling and stochastic gradient ascent. 

The \emph{importance weighted autoencoder} \citep[\IWAE,][]{burda:2016} extends this idea further by optimizing a similar but improved ELBO, where the expectation is taken with respect to the law of $N$ independent and $q_\thprop$-distributed random variables, denoted by $q_\thprop^{\tensprod N}$: %, \ie
\foraistats{\begin{multline} %\label{eq:elbo:iwae}
	\elbo^\iwae(\thprop, \thp)\eqdef\E_{q_\thprop^{\tensprod N}}\left[\log\left(\frac{1}{N}\sum_{i=1}^{N}\frac{p_\thp(X_{0:T}^i,y_{0:T})}{q_\thprop(X_{0:T}^i\mid y_{0:T})}\right)\right] \label{eq:sis}
	\\\le\log\E_{q_\thprop^{\tensprod N}}\left[\frac{1}{N}\sum_{i=1}^{N}\frac{p_\thp(X_{0:T}^i,y_{0:T})}{q_\thprop(X_{0:T}^i\mid y_{0:T})}\right]
	\\=\log \E_{q_\thprop}\left[ \frac{p_\thp(X_{0:T},y_{0:T})}{q_\thprop(X_{0:T}\mid y_{0:T})} \right]= \log p_\thp(y_{0:T}).
\end{multline}}
\forarxiv{\begin{multline} %\label{eq:elbo:iwae}
		\elbo^\iwae(\thprop, \thp)\eqdef\E_{q_\thprop^{\tensprod N}}\left[\log\left(\frac{1}{N}\sum_{i=1}^{N}\frac{p_\thp(X_{0:T}^i,y_{0:T})}{q_\thprop(X_{0:T}^i\mid y_{0:T})}\right)\right] \label{eq:sis}
		\le\log\E_{q_\thprop^{\tensprod N}}\left[\frac{1}{N}\sum_{i=1}^{N}\frac{p_\thp(X_{0:T}^i,y_{0:T})}{q_\thprop(X_{0:T}^i\mid y_{0:T})}\right]
		\\=\log \E_{q_\thprop}\left[ \frac{p_\thp(X_{0:T},y_{0:T})}{q_\thprop(X_{0:T}\mid y_{0:T})} \right]= \log p_\thp(y_{0:T}).
\end{multline}}
Again, the bound follows from Jensen's inequality and the fact that the average that is logarithmised is an unbiased estimator of the likelihood. The {\IWAE} provides an improvement on standard VI as $\elbo^\iwae$ provides a tighter lower bound on the likelihood $\log p(y_{0:T})$ and gets, assuming bounded weights, arbitrarily close to the same as $N$ tends to infinity \citep[Theorem~1]{burda:2016}. 
%Still, a problem with the {\IWAE} is that, when approximating $\elbo^\iwae$ by standard Monte Carlo, if $T$ is moderately large and $q_\thprop(x_{0:T}\mid y_{0:T})$ is not sufficiently close to $p_\thp(x_{0:T}\mid y_{0:T})$ we would have high skewness of the weights. 
Still, %{\IWAE} is problematic in that 
once $T$ is reasonably large and $q_\thprop(x_{0:T} \mid y_{0:T})$ is not sufficiently close to $p_\thp(x_{0:T}\mid y_{0:T})$, approximating $\elbo^\iwae$ by standard Monte Carlo implies generally high variance; indeed, in the extreme case, the highly skewed distribution of the terms of the likelihood estimator will effectively reduce {\IWAE} to standard VI. 
%all the terms in the sum in \eqref{eq:sis} will be negligible compared to the largest one, effectively reducing {\IWAE} to standard VI. 
%In the extreme case we might observe all the terms in the sum but one being negligible with respect to a dominating term, reducing in practice to standard VI. \noteJO{Maybe it's better to speak about skewness among the terms?} \noteAM{Changed.} 
This degeneracy problem can be counteracted by replacing standard Monte Carlo approximation with resampling-based SMC.   
%This %well known curse-of-dimensionality 
%issue may be counteracted in the SMC framework by the resampling operation described in Line~\ref{line:res} of Algorithm~\ref{algo:pf}. 
For this reason, \citet{le:2018}, \citet{maddison:2017} and \citet{naesseth:2018} all define an ELBO similar to \eqref{eq:sis}, but where the expectation is taken with respect to the law $\partlaw{\thprop}$ of the random variables $(\epart{0}{i})_{i=1}^N$ and $(\epart{t}{i}, \I{t}{i})_{i = 1}^N$, $t \in \{1, \ldots, T \}$, generated by the particle filter: 
%\noteJO{Shouldn't this law also depend on $\theta$?} \noteAM{In practice yes because of the resampling weights, but I wanted to highlight that given a model, the design (and law) of the SMC algorithm is determined by the proposal and its parameters.} 
%Given the temporal dependence of the data, we may let $q_\thprop$ factorize and express it as $q_\thprop(x_{0:T}\mid y_{0:T})=\propdens{\thprop}^0(x_0\mid y_0)\prod_{t=0}^{T-1}\propdens{\thprop}(x_{t+1}\mid x_t,y_{t+1})$.  
%One then defines 
$$%\begin{multline}
\elbo^\smc(\thprop, \thp) \eqdef \E_{\partlaw{\thprop}} \left[ \log\left(  \prod_{t = 0}^T \frac{1}{N} \wgtsum{t} \right) \right],  
$$%\end{multline}
%\begin{multline}
%	\elbo^\smc(\thprop, \thp) \eqdef \E_{\partlaw{\thprop}} \left[ \log\left( 
%	\vphantom{\frac{1}{N}\sum_{i=1}^{N}\frac{\hiddens{\thp}(\epart{t+1}{i}\mid \epart{t}{\I{t+1}{i}})\emdens{\thp}(y_{t+1}\mid \epart{t+1}{i})}{\propdens{\thprop}(\epart{t+1}{i}\mid \epart{t}{\I{t+1}{i}},y_{t+1})}} 
%		\frac{1}{N} \sum_{i=1}^{N}\frac{\hiddens{\thp}^0(\epart{0}{i})\emdens{\thp}(y_{0}\mid \epart{0}{i})}{\propdens{\thprop}^0(\epart{0}{i}\mid y_0)}\right.\right.
%	\\\left.\left.\prod_{t=0}^{T-1} \frac{1}{N}\sum_{i=1}^{N}\frac{\hiddens{\thp}(\epart{t+1}{i}\mid \epart{t}{\I{t+1}{i}})\emdens{\thp}(y_{t+1}\mid \epart{t+1}{i})}{\propdens{\thprop}(\epart{t+1}{i}\mid \epart{t}{\I{t+1}{i}},y_{t+1})} \right)\right],
%\end{multline}
which is again a lower bound on $\log p_\thp(x_{0:T})$ %(and to $\elbo^\vi$ too \cite[Theorem~1]{naesseth:2018})
as the argument of the logarithm is an unbiased estimator of the likelihood \citep[see, \eg,][Proposition~16.3]{chopin:papaspiliopoulos:2020}. 

Like in standard VI, the maximization of $\elbo^\smc$ is carried out by alternately (1) processing the given data batch $y_{0:t}$ with the particle filter and (2) taking a stochastic gradient ascent step. The latter involves the differentiation of $\elbo^\smc$ with respect to $(\thprop,\thp)$, which can be carried through using the reparameterization trick \citep{kingma:welling:2014}. More specifically, it is assumed that $\propdens{\thprop}(\cdot \mid x, y)$ is reparameterizable in the sense that there exists some auxiliary random variable $\auxrv$, taking on values in some measurable space $(\set{E}, \alg{E})$ and having distribution $\indmeas$ on $(\set{E}, \alg{E})$ (the latter not depending on $\thprop$), and some function $\repfunc{\thprop}$ on $\set{X} \times \set{Y} \times \set{E}$, parameterized by $\thprop$, such that 
for every $(x, y) \in \set{X} \times \set{Y}$, the pushforward distribution $\indmeas \circ \repfunc{\thprop}^{-1}(x, y, \cdot)$ coincides with that governed by $\propdens{\thprop}(\cdot \mid x, y)$. In addition, importantly, for any given argument $(x, y, \varepsilon)$, $\repfunc{\thprop}(x, y, \varepsilon)$ is assumed to be differentiable with respect to $\thprop$. A similar reparameterization assumption is made for some initial proposal distribution denoted by $\propdens{0,\thprop}$. The previous assumptions allow us to reparameterize Algorithm~\ref{algo:pf} by splitting the procedure on Line~4 into two suboperations: first sampling $\auxrv_{t+1}^i \sim \indmeas$ and then letting $\epart{t+1}{i} \gets \repfunc{\thprop}(\epart{t}{\I{t+1}{i}}, y_{t+1}, \auxrv_{t+1}^i)$. After this, the importance weights can be calculated as explicit functions of $(\thprop,\thp)$ according to 
%; then we may define $\wgtfunc{0}{i}(\thprop,\thp)\eqdef\hiddens{\thp}^0(\repfunc{\thprop}(y_0,\auxrv_{0}^i))\emdens{\thp}(y_0\mid\repfunc{\thprop}(y_0,\auxrv_{0}^i))/\propdens{\thprop}^0(\repfunc{\thprop}(y_0,\auxrv_{0}^i)\mid y_0)$ and
\foraistats{\begin{multline}\label{eq:wgtfunc}
	\wgtfunc{t+1}{i}(\thprop,\thp)\eqdef\hiddens{\thp}(\repfunc{\thprop}(\epart{t}{\I{t+1}{i}},y_{t+1},\auxrv_{t+1}^i)\mid\epart{t}{\I{t+1}{i}})
	\\\times\frac{\emdens{\thp}(y_{t+1}\mid \repfunc{\thprop}(\epart{t}{\I{t+1}{i}},y_{t+1},\auxrv_{t+1}^i))}{\propdens{\thprop}(\repfunc{\thprop}(\epart{t}{\I{t+1}{i}},y_{t+1},\auxrv_{t+1}^i)\mid\epart{t}{\I{t+1}{i}},y_{n+1})}
\end{multline}}
\forarxiv{\begin{equation}\label{eq:wgtfunc}
		\wgtfunc{t+1}{i}(\thprop,\thp)\eqdef\hiddens{\thp}(\repfunc{\thprop}(\epart{t}{\I{t+1}{i}},y_{t+1},\auxrv_{t+1}^i)\mid\epart{t}{\I{t+1}{i}})\frac{\emdens{\thp}(y_{t+1}\mid \repfunc{\thprop}(\epart{t}{\I{t+1}{i}},y_{t+1},\auxrv_{t+1}^i))}{\propdens{\thprop}(\repfunc{\thprop}(\epart{t}{\I{t+1}{i}},y_{t+1},\auxrv_{t+1}^i)\mid\epart{t}{\I{t+1}{i}},y_{n+1})}
\end{equation}}
and, at initialization, $\epart{0}{i} \gets \repfunc{\thprop}(y_0,\auxrv_{0}^i) $ and $\wgtfunc{0}{i}(\thprop,\thp)\eqdef\hiddens{0}(\repfunc{\thprop}(y_0,\auxrv_{0}^i))\emdens{\thp}(y_0\mid\repfunc{\thprop}(y_0,\auxrv_{0}^i))/\propdens{0,\thprop}(\repfunc{\thprop}(y_0,\auxrv_{0}^i)\mid y_0)$.
%\noteJO{I don't really understand the notation here. I would use $\wgt{t + 1}{i}(\thprop,\thp)$ and $\wgtsum{t + 1}(\thprop,\thp)$ I think.}

%exists a function $\repfunc{\thprop}$, differentiable with respect to $\thprop$, and some $\auxrv\sim\indmeas(\cdot)$, where $\indmeas(\cdot)$ is a distribution not depending on $\thprop$, such that $\repfunc{\thprop}(x_t,y_{t+1},\auxrv)$ has the same distribution as any random variable sampled from $\propdens{\thprop}(\cdot\mid x_t,y_{t+1})$, for all given $(\thprop,x_t,y_{t+1})$. A similar argument can be made for the reparameterization of $\propdens{\thprop}^0$, without conditioning to any previous state.

\citet{le:2018}, \citet{maddison:2017} and \citet{naesseth:2018} have shown that the Monte Carlo approximation of $\nabla_{\thprop,\thp} \elbo^\smc(\thprop,\thp)$ can, in order to avoid unmanageable variance, be advantageously carried through by targeting the ``surrogate'' gradient 
%Before presenting such approximation of the gradient of the SMC ELBO, which we will denote by $\gradapprox_T(\thprop,\thp)$, we introduce some further notation. First of all, we reparameterize Algorithm~\ref{algo:pf} by splitting the procedure of Line~\ref{line:pfprop} into sampling $\auxrv_{t+1}^i\sim\indmeas(\cdot)$ and then letting $\epart{t+1}{i}\gets \repfunc{\thprop}(\epart{t}{\I{t+1}{i}},y_{t+1},\auxrv_{t+1}^i)$; then we may define $\wgtfunc{0}{i}(\thprop,\thp)\eqdef\hiddens{\thp}^0(\repfunc{\thprop}(y_0,\auxrv_{0}^i))\emdens{\thp}(y_0\mid\repfunc{\thprop}(y_0,\auxrv_{0}^i))/\propdens{\thprop}^0(\repfunc{\thprop}(y_0,\auxrv_{0}^i)\mid y_0)$ and
%\begin{multline}\label{eq:wgtfunc}
%	\wgtfunc{t+1}{i}(\thprop,\thp)\eqdef\hiddens{\thp}(\repfunc{\thprop}(\epart{t}{\I{t+1}{i}},y_{t+1},\auxrv_{t+1}^i)\mid\epart{t}{\I{t+1}{i}})
%	\\\times\frac{\emdens{\thp}(y_{t+1}\mid \repfunc{\thprop}(\epart{t}{\I{t+1}{i}},y_{t+1},\auxrv_{t+1}^i))}{\propdens{\thprop}(\repfunc{\thprop}(\epart{t}{\I{t+1}{i}},y_{t+1},\auxrv_{t+1}^i)\mid\epart{t}{\I{t+1}{i}},y_{n+1})},
%\end{multline}
%for $t\ge0$ and $i=1,\dots,N$. Then $\gradapprox_T(\thprop,\thp)$ is defined as follows:
\begin{equation}\label{eq:def_grad}
\gradapprox_T(\thprop,\thp) \eqdef \E_{\partlaw{\thprop}} \left[ \nabla_{(\thprop,\thp)} \log\left(  \prod_{t = 0}^T \frac{1}{N} \wgtsum{t}(\thprop,\thp) \right) \right], 
\end{equation}
%\begin{multline}\label{eq:def_grad}
%	\gradapprox_T(\thprop,\thp)\eqdef\E_{\smc_{\repfunc{\thprop},\indmeas}}\left[\nabla_{(\thprop,\thp)}\log\left(\frac{1}{N}\sum_{i=1}^{N}\wgtfunc{0}{i}(\thprop,\thp)\right.\right.
%	\\\left.\left.\times\prod_{t=0}^{T-1}\frac{1}{N}\sum_{i=1}^{N}\wgtfunc{t+1}{i}(\thprop,\thp)\right)\right].
%\end{multline}
where $\partlaw{\thprop}$ now corresponds to the law of the random variables $(\auxrv_0^i)_{i=1}^N$ and $(\auxrv_t^i, \I{t}{i})_{i = 1}^N$, $t \in \{1, \ldots, T \}$, generated by the particle filter using the reparameterization trick.  
%where $\E_{\smc_{\repfunc{\thprop},\indmeas}}$ indicates the expectation with respect to the random variables $(\auxrv_{0:T},\I{1:T}{i})_{i=1}^N$ generated by the particle filter under the reparameterization trick. 
The approximate gradient $\gradapprox_T(\thprop,\thp)$ is indeed different from $\nabla_{(\thprop,\thp)} \elbo^\smc(\thprop, \thp)$ in that the latter contains one additional term corresponding to the expectation of the product of $\nabla_{(\thprop,\thp)} \log \partlaw{\thprop}$ and the logarithm of the unbiased likelihood estimator \citep[see][Appendix~A, for details]{le:2018}.
%. there is an additive term to the given approximation $\gradapprox_T(\thprop,\thp)$; 
Nonetheless, since, as demonstrated by \citet{le:2018}, \citet{maddison:2017} and \citet{naesseth:2018}, this term generally does not make a significant contribution to the gradient and is also very difficult to estimate with reasonable accuracy, we proceed as in the mentioned works and simply discard the same. \citet{roeder:2017}, \citet{tucker:2018} and \citet{finke:thiery:2019} discussed biased gradient approximation in the framework of the {\IWAE}, but we are not aware of any in-depth analyses in the context of {\VSMC}.
%s demonstrated in \cite{le:2018,maddison:2017,naesseth:2018}, estimating this additional term with reasonable accuracy appears to be very difficult, and since ; and  this add
%Nonetheless, this additional expression has been demonstrated to cause high variability to the estimation, without significant improvements, hence we discard it as in \cite{le:2018,maddison:2017,naesseth:2018}.

As mentioned earlier, the variational SMC ({\VSMC}) approach described above is an offline method in the sense that each parameter update requires $\gradapprox_T(\thprop,\thp)$ to be estimated by processing the entire data batch $y_{0:T}$ with the particle filter. %These updates are repeated an arbitrary number of times until no more improvement of the ELBO is observed. 
However, in many applications it is of utmost importance to be able to update both parameter estimates and proposal distributions in real time, and in the next section we will therefore provide an online extension of {\VSMC} to the setting where the data become available sequentially.   
%In Section~\ref{sec:method} we describe how this method can be translated to an online framework.

\section{Online Implementation of \VSMC}\label{sec:method}
%In Section~\ref{sec:background} we briefly described the batch version of {\VSMC}; 
In this section our primary goal is to learn, for a given stream $(y_t)_{t\ge 0}$ of observations, the amortized proposal and model parameters as the particles evolve and new observations become available. By rewriting \eqref{eq:def_grad} as 
$$
		\gradapprox_T(\thprop,\thp) = \E_{\partlaw{\thprop}}\left[\sum_{t=0}^{T}\nabla_{(\thprop,\thp)}\log\wgtsum{t}(\thprop,\thp)\right] 
$$
we notice that the computation of the gradient is distributed over time, making it possible to adapt the method to the online setting. More precisely, in our scheme, the given parameters $(\thp_t, \thprop_t)$ are updated as 
\begin{align}
\thprop_{t+1} &\gets \thprop_t + \gamma_{t+1}^\thprop \nabla_{\thprop}\log\wgtsum{t+1}(\thprop_t,\thp_t), \label{eq:update:lambda} \\ 
\thp_{t+1} &\gets \thp_t + \gamma_{t+1}^\thp \nabla_\thp \log \wgtsum{t+1} (\thprop_{t+1},\thp_t), \label{eq:update:theta}
\end{align}
where $(\gamma_t^\thprop)_{t \in \nset_{>0}}$ and $(\gamma_t^\thp)_{t \in \nset_{>0}}$ are given step sizes (learning rates). 
%Thus, we propose to perform one stochastic optimization step at each iteration of the particle filter by, at iteration $t$, given parameter estimates $(\thp_t, \thprop_t)$, moving in the direction of the gradient of the logarithm of the total weight $\wgtsum{t+1}(\thprop_t,\thp_t)$ of the particles induced by the current observation $y_{t + 1}$. 
%As clear from the pseudocode displayed in Algorithm~\ref{algo:ovpf}, it is not necessary to store the particle trajectories, which would be infeasible in the long run; thus, we may skip the operation on Line~\ref{line:attach_part} in Algorithm~\ref{algo:pf}, resulting in an online algorithm with memory requirements that remain uniformly limited in time. In Section~\ref{sec:theory} we provide a rigorous theoretical justification of (a slightly modified version of) Algorithm~\ref{algo:ovpf}.%, presented in Algorithm~1 of the supplement; 
%In particular, we show that our stochastic-approximation scheme solves the same problem of {\VSMC} in the limit case $t\to\infty$ and where the gradient is time normalized. 

A pseudocode for our algorithm, which we refer to as \emph{online variational SMC} ({\OVSMC}), is displayed in Algorithm~\ref{algo:ovpf}, from which it can be seen that the updates of $\thprop$ and $\thp$ on Line~7 and Line~14 are based on two distinct sampling steps with different sample sizes $L$ and $N$, respectively.  
%Finally, we also intend to clarify one point: since we compute the gradient with respect to the joint vector of parameters $(\thprop,\thp)$, we should point out two distinct behaviors of the estimated gradient of the ELBO depending on the number of particles $N$. 
Indeed, as pointed out by \citet{le:2018} and \citet{zhao:2021}, the quantity $\log \wgtsum{t + 1}(\thprop,\thp)$ is a biased but consistent estimator of the log-predictive likelihood $\log p_\thp(y_{t + 1} \mid y_{0:t})$ and will therefore, regardless of the proposal used, be arbitrarily close to this quantity as the number of particles increases. In contrast to the estimation of $\nabla_\thp \log \wgtsum{t}(\thprop, \thp)$, this is generally problematic in the estimation of $\nabla_\thprop \log \wgtsum{t}(\thprop, \thp)$. 
%This is not an issue when estimating the gradient with respect to $\thp$, but it is problematic when estimating  it with respect to $\thprop$. 
In fact, a large number of particles reduces the signal-to-noise ratio of the estimator of the latter gradient, up to a point where it reduces to pure noise. For this reason, we take, in the spirit of the \emph{alternating ELBOs} strategy of \citet[Section~4.1, in which the authors consider {\IWAE} and {\VSMC} ELBOs with alternating sample sizes]{le:2018} an approach where  $\thprop$ and $\thp$ are updated through two distinct optimization steps, the one for $\thprop$ with a  small number $L$ of particles, typically less than ten, and the one for $\thp$ with a possibly large sample size $N$.  
%(in which they alternate ELBOs from IWAE and VSMC with different numbers of samples), in Algorithm~\ref{algo:ovpf}, 
%an approach where the parameters $\thprop$ and $\thp$ are updated where at each iteration of the particle filter we have two distinct optimization steps: one with few $L$ particles, typically less than ten, for $\thprop$, the other with a possibly large number $N$ for $\thp$. 
%The sequences $(\gamma_{t}^\thprop)_{t\ge 1}$ and  $(\gamma_{t}^\thp)_{t\ge 1}$ represent the learning rates of the optimizers of the proposal and the model parameters, respectively.

Appealingly, as clear from Algorithm~\ref{algo:ovpf}, the method is based only on particle approximation of the filter distribution flow and  therefore does not require saving the trajectories of the particles. This 
%which would be infeasible in the long run,
%; thus, we may skip the operation on Line~\ref{line:attach_part} in Algorithm~\ref{algo:pf}, 
results in an online algorithm with memory requirements that remain uniformly limited in time and are, just like the computational complexity of the algorithm, linear in the number $N$ of particles. In Section~\ref{sec:theory} we provide a rigorous theoretical justification of (a slightly modified version of) Algorithm~\ref{algo:ovpf}. %, presented in Algorithm~1 of the supplement; 
In particular, we show that \eqref{eq:update:lambda}--\eqref{eq:update:theta}  form a classical Robbins--Monro scheme \citep{robbins:monro:1951} with state-dependent Markov noise targeting a mean field corresponding to an `asymptotic' (as the number $t$ of data tends to infinity) {\VSMC}. 
%solves the same problem of {\VSMC} in the limit case $t\to\infty$ and where the gradient is time normalized. 

%\begin{algorithm}[htb]
%	\caption{Online Variational SMC ({\OVSMC})}\label{algo:ovpf}
%	\begin{algorithmic}[1]
%		\Statex \textbf{Input}: $(\epart{t}{i},\wgt{t}{i})_{i=1}^N, y_{t+1}, \thp_t,\thprop_t$.
%		\For{$i \gets 1,\dots, L$}
%		\State draw $\I{t+1}{i}\sim \catdist((\wgt{t}{\ell})_{\ell=1}^N)$;
%		\State draw $\auxrv_{t+1}^i\sim \indmeas$;
%		\State compute $\wgtfunc{t+1}{i}(\thprop_t,\thp_t)$ according to \eqref{eq:wgtfunc};
%		\EndFor
%		\State set $\thprop_{t+1}\gets\thprop_t+\gamma_{t+1}^\thprop\nabla_{\thprop}\log\wgtsum{t+1}(\thprop_t,\thp_t)$;\label{line:thprop}
%		\For{$i \gets 1,\dots,N$}\label{lin:startloop}
%		\State draw $\I{t+1}{i}\sim \catdist((\wgt{t}{\ell})_{\ell=1}^N)$;
%		\State draw $\auxrv_{t+1}^i\sim \indmeas$;\label{line:eps}
%		\State set $\epart{t+1}{i}\gets\repfunc{\thprop_{t+1}}(\epart{t}{\I{t+1}{i}},y_{t+1},\auxrv_{t+1}^i)$;
%		\State compute $\wgtfunc{t+1}{i}(\thprop_{t+1},\thp_t)$ according to \eqref{eq:wgtfunc};%		\State set $\wgt{t+1}{i}\gets\wgtfunc{t+1}{i}(\thprop_{t+1},\thp_t)$;
%		\EndFor
%		\State set $\thp_{t+1}\gets\thp_t+\gamma_{t+1}^\thp\nabla_\thp\log\wgtsum{t+1}(\thprop_{t+1},\thp_t)$;\label{line:endalgo}
%		\State \Return $(\epart{t+1}{i},\wgt{t+1}{i})_{i=1}^N,\thp_{t+1},\thprop_{t+1}$.
%	\end{algorithmic}
%\end{algorithm}

\begin{algorithm}[htb]
	\caption{Online Variational SMC ({\OVSMC})}\label{algo:ovpf}
	\begin{algorithmic}[1]
		\STATE {\bfseries Input:} $(\epart{t}{i},\wgt{t}{i})_{i=1}^N, y_{t+1}, \thp_t,\thprop_t$.
		\FOR{$i \gets 1,\dots, L$}
		\STATE draw $\I{t+1}{i}\sim \catdist((\wgt{t}{\ell})_{\ell=1}^N)$;
		\STATE draw $\auxrv_{t+1}^i\sim \indmeas$;
		\STATE compute $\wgtfunc{t+1}{i}(\thprop_t,\thp_t)$ according to \eqref{eq:wgtfunc};
		\ENDFOR
		\STATE set $\thprop_{t+1}\gets\thprop_t+\gamma_{t+1}^\thprop\nabla_{\thprop}\log\wgtsum{t+1}(\thprop_t,\thp_t)$;\label{line:thprop}
		\FOR{$i \gets 1,\dots,N$}\label{lin:startloop}
		\STATE draw $\I{t+1}{i}\sim \catdist((\wgt{t}{\ell})_{\ell=1}^N)$;
		\STATE draw $\auxrv_{t+1}^i\sim \indmeas$;\label{line:eps}
		\STATE set $\epart{t+1}{i}\gets\repfunc{\thprop_{t+1}}(\epart{t}{\I{t+1}{i}},y_{t+1},\auxrv_{t+1}^i)$;
		\STATE compute $\wgtfunc{t+1}{i}(\thprop_{t+1},\thp_t)$ according to \eqref{eq:wgtfunc};%		\State set $\wgt{t+1}{i}\gets\wgtfunc{t+1}{i}(\thprop_{t+1},\thp_t)$;
		\ENDFOR
		\STATE set $\thp_{t+1}\gets\thp_t+\gamma_{t+1}^\thp\nabla_\thp\log\wgtsum{t+1}(\thprop_{t+1},\thp_t)$;\label{line:endalgo}
		\STATE {\bfseries return} $(\epart{t+1}{i},\wgt{t+1}{i})_{i=1}^N,\thp_{t+1},\thprop_{t+1}$.
	\end{algorithmic}
\end{algorithm}

\section{Theoretical Results}\label{sec:theory}
In this section we study the limiting behavior of {\OVSMC} and discuss its connection with batch {\VSMC}.
%, by inspecting the gradient approximation for $t$ arbitrarily large. 
As explained in Section~\ref{sec:method}, the purpose of the double optimization steps in Algorithm~\ref{algo:ovpf} is to improve the performance of the algorithm in practical use;  
%We first remark that the double optimization steps of Algorithm~\ref{algo:ovpf} is intended to the practical use, as explained in Section~\ref{sec:method}; 
however, here, for simplicity, we move the parameters $\thprop$ into $\thp$, resulting in the latter also containing algorithmic parameters. Therefore, the algorithm that we theoretically analyze corresponds to Lines~8--14, with reparameterization function $\repfunc{\thp_t}$ and $\wgtfunc{t+1}{i}$ depending only on $\thp_t$.  Furthermore, for technical reasons related to the ergodicity of an extended Markov chain that we will define below and the Lipschitz continuity (in $\thp$) of its transition kernel, the analyzed algorithm repeats twice the generation of the auxiliary variable on Line~10, where one variable is used to estimate the gradient and the other to particle propagation.  
%(in which case we do not actually need the reparameterization trick).  
All details are provided in Appendix~\ref{appendix:a}, where the modified procedure is displayed in Algorithm~\ref{algo:ovpf_theory} and we also provide all proofs.

Let, for $t \in \nset$, 
%The Markov chain $(\z_t^{\thp})_{t \in \nset}$, where 
$
\z_t \coloneqq (X_t, Y_t, (\epart{t-1}{\I{t}{i}})_{i=1}^N, (\auxrv_{t}^i)_{i=1}^N) 
$, 
where $(\epart{t-1}{\I{t}{i}})_{i=1}^N$ and  $(\auxrv_t^i)_{i=1}^N$ are generated according to the modified version of Algorithm~\ref{algo:ovpf}. Then $(\z_t)_{t \in \nset}$ is a state dependent Markov chain with transition kernel $T_\thp$ (described in the supplement), in the sense that given $\z_{0:t}$, $\z_{t + 1}$ is distributed according to $T_{\thp_t}(\z_t, \cdot)$ (note that $\thp_t$ is deterministic function of $\z_{0:t}$). Let $\prob$ denote the law of $(\z_t)_{t \in \nset}$ when initialized as described previously.  
\begin{assumption}\label{assum:ssm}
	The data $(y_t)_{t \in \nset}$ is the output of an SSM $(X_t, Y_t)_{t \in \nset}$ on $(\set{X} \times \set{Y}, \alg{X} \tensprod \alg{Y})$ with state and observation transition densities $\hiddenstrue(x_{t + 1} \mid x_t)$ and $\emdenstrue(y_t \mid x_t)$, respectively.  
\end{assumption}
\begin{assumption}\label{assum:strongmixing}
	The transition densities $\hiddenstrue$, $\hiddens{\thp}$, $\emdens{\thp}$ and $\propdens{\thp}$ are uniformly bounded from above and below (in all their arguments as well as in $\thp$). 
\end{assumption}

The strong mixing assumptions of Assumption~\ref{assum:strongmixing} are standard in the literature and point to applications where the state space $\set{X}$ is a compact set.  

\begin{proposition}\label{prop:ergomain}
	Let Assumptions~\ref{assum:ssm}--\ref{assum:strongmixing} hold. Then for every $\thp \in \Theta$, the canonical Markov chain $(\z_t^\thp)_{t \in \nset}$ induced by $T_\thp$ is uniformly ergodic and admits a stationary distribution $\statmeas{\thp}$. 
\end{proposition}
 Now, letting $\grad{\thp}(\z_t) \eqdef \nabla_{\thp} \log \wgtsum{t}(\thp)$, we may define the mean field 
 %so-called \emph{mean field}, depending on $\thp$ and $N$,
$
	\mf{\thp}\eqdef\int \grad{\thp}(z)\,\statmeas{\thp}(dz)
$
(depending implicitly on the sample size $N$). 
Note that by the law of large numbers for ergodic Markov chains it holds, a.s., that $\lim_{t\to \infty} \sum_{s=0}^{t}\grad{\thp}(\z_s^{\thp}) / t =\mf{\thp}$.
%\begin{equation}\label{eq:ergodic_lln}
%	\lim_{t\to \infty} \frac{1}{t}\sum_{s=0}^{t}\grad{\thp}(\z_s^{\thp})=\mf{\thp}\quad \prob\mathrm{-a.s.}.
%\end{equation}
Thus, since ${\VSMC}$ estimates the gradient $\gradapprox_t(\thp)$ by $\sum_{s=0}^t \grad{\thp}(\z_s^{\thp})$, finding a zero of the mean field $\mf{\thp}$ can be considered equivalent to applying {\VSMC} to an infinitely large batch of observations from the model.  
% = \E_{\partlaw{\thprop}}[\sum_{s = 0}^t \grad{\thp}(\z_s^{\thp})]$ 
%Since for every $t$ we could write $\gradapprox_t(\thp)=\E_{\partlaw{\thprop}}\left[\sum_{s=0}^{t}\grad{\thp}(\z_s^{\thp})\right]$, then 
%we may imagine to solve an ideal standard VSMC problem with an arbitrarily large sequence of data and time-normalized gradient: due to \eqref{eq:ergodic_lln}, as $t\to \infty$, such problem would then correspond to finding a zero of the mean field $\mf{\thp}$. 
From this perspective, it is interesting to note that our proposed {\OVSMC} algorithm is a Robbins--Monro scheme targeting $\mf{\thp}$, with updates given by $\thp_{t + 1} \gets \thp_t+\gamma_{t+1}^\thp \grad{\thp_t}(\z_{t + 1})$ and $Z_{t + 1} \sim T_{\thp_t}(Z_t, \cdot)$, $t \in \nset$. Moreover, as established by our main result, Theorem~\ref{thm:main}, the scheme produces a parameter sequence $(\theta_t)_{t \in \nset}$ making the gradient arbitrarily close to zero (in the $L_2$ sense). We preface this result by some further assumptions. 
\begin{assumption}\label{assum:lipsch}
	The gradients $\nabla_\thp \hiddens{\thp}$, $\nabla_\thp \emdens{\thp}$ and $\nabla_\thp \propdens{\thp}$, and their compositions with the reparameterization function $\repfunc{\thp}$, are all uniformly bounded and Lipschitz. 
	%for all $\thp\in\parspace$. Moreover, all these densities and their gradients, even when composed with the reparameterization function $\repfunc{\thp}$, satisfy the Lipschitz condition in $\thp$ for some positive constant.
\end{assumption}

%\begin{assumption}\label{assum:lipsch}
%	The norms of the gradients of the transition densities $\hiddens{\thp}$, $\emdens{\thp}$ and $\propdens{\thp}$, when composed with the reparameterization function $\repfunc{\thp}$, are uniformly bounded for all $\thp\in\parspace$. Moreover, all these densities and their gradients, even when composed with the reparameterization function $\repfunc{\thp}$, satisfy the Lipschitz condition in $\thp$ for some positive constant.
%\end{assumption}

\begin{assumption}\label{assum:lyap}
	There exists a bounded function $V$ on $\Theta$ (the \emph{Lyapunov} function) such that $h=\nabla_\thp V$.
\end{assumption}

\begin{assumption}\label{assum:step}
For every $t \in \nsetpos$, $\gamma_{t+1}^\thp \le \gamma_t^\thp$. In addition, there exist constants $a > 0$, $a' > 0$ and $c \in (0, 1)$ such that for every $t$, $
		\gamma_t^\thp \le a \gamma_{t + 1}^\thp$, $ \gamma_t^\thp - \gamma_{t + 1}^\thp \le a' (\gamma_{t+1}^\thp)^2$ and $ \gamma_1 \le c$.
%	for some $a,a'>0$, and some $c<1$ depending on $a,a'$, and on the uniform bounds and constants described in Assumptions~\ref{assum:strongmixing} and \ref{assum:lipsch}.
\end{assumption}

%\begin{assumption}\label{assum:step}
%	The step sizes $(\gamma_{t}^\thp)_{t\ge 1}$ are such that for every $t \in \nsetpos$, 
%	\begin{equation}
%		\gamma_{t+1}\le\gamma_{t},\quad\gamma_{t}\le a\gamma_{t+1},\quad\gamma_{t}-\gamma_{t+1}\le a' \gamma_{t+1}^2,\quad\gamma_{1}\le c,
%	\end{equation}
%	for some $a,a'>0$, and some $c<1$ depending on $a,a'$, and on the uniform bounds and constants described in Assumptions~\ref{assum:strongmixing} and \ref{assum:lipsch}.
%\end{assumption}
%\begin{theorem}\label{thm:main}
%	Let Assumptions~\ref{assum:strongmixing}, \ref{assum:lipsch}, \ref{assum:lyap} and \ref{assum:step} hold and let $\tau$ be a random variable on $\{0, \dots, t\}$, $t \ge 1$, such that $\prob(T=t')\eqdef(\sum_{s=0}^t\gamma_{s+1}^\thp)^{-1}\gamma_{t'+1}^\thp$, then there exist constants $b,b'>0$, possibly depending on $V$ and on the bounds and constants described in Assumptions~\ref{assum:strongmixing} and \ref{assum:lipsch}, such that
%	\begin{equation}
%		\E[\norm{\mf{\thp_T}}^2]\le \frac{b + b'\sum_{s=0}^t(\gamma_{s+1}^\thp)^2}{\sum_{s=0}^t\gamma_{s+1}^\thp}.
%	\end{equation}
%\end{theorem}
\begin{theorem} \label{thm:main}
	Let Assumptions~\ref{assum:ssm}, \ref{assum:strongmixing}, \ref{assum:lipsch}, \ref{assum:lyap} and \ref{assum:step} hold. Then there exist constants $b > 0$ and $b' > 0$, possibly depending on $V$, 
	% and on the bounds and constants described in Assumptions~\ref{assum:strongmixing} and \ref{assum:lipsch}, 
	such that for every $t \in \nset$, 
	\begin{equation}
		\E[\norm{\mf{\thp_\tau}}^2] \le \frac{b + b'\sum_{s=0}^t(\gamma_{s+1}^\thp)^2}{\sum_{s=0}^t\gamma_{s+1}^\thp}, 
	\end{equation}
	where $\tau \sim \catdist((\gamma_{s + 1}^\thp)_{s = 0}^t)$ is a random variable on $\{0, \dots, t\}$. %, $t \ge 1$, such that $\prob(T=t')\eqdef(\sum_{s=0}^t\gamma_{s+1}^\thp)^{-1}\gamma_{t'+1}^\thp$, 
\end{theorem}
%The previous result is derived using theory of \cite[Theorem 2]{karimi:2019}. and we immediately see that if for every $t\ge1$ we choose $\gamma_t^\thp=\mathcal{O}(1/\sqrt{t})$, we may push the expected mean field arbitrarily close to zero.
\begin{corollary}\label{corollary:decay}
	Let the assumptions of Theorem~\ref{thm:main} hold and let, for every $t \in \nsetpos$, $\gamma_t^\thp = c / \sqrt{t}$, where $c > 0$ is given in Assumption~\ref{assum:step}. Then $\E[\norm{\mf{\thp_\tau}}^2] = \mathcal{O}(\log t / \sqrt{t})$. 
%	Then for every $t \in \nsetpos$,
%	\begin{equation}
%		\E[\norm{\mf{\thp_\tau}}^2] = \mathcal{O}\left(\frac{\log t}{\sqrt{t}}\right).
%	\end{equation}
\end{corollary}
%Hence, for $t$ arbitrarily large both time-normalized VSMC and OVSMC aim to find a zero of the same mean field and thus they solve the same optimization problem.

%\noteJO{In my view, it is not sufficiently clear that the method is nothing but a Robbins–Monro procedure with state dependent Markov noise. This has to be perfectly clear.}
%
%\noteJO{Provide the example $\gamma_t \propto 1/t$, in which case the bound is $O(\log t / \sqrt{t})$.}
%
%\noteJO{I think this discussion should be given already when the mean field $h(\theta)$ is defined. Indeed, you should provide an additional theorem stating the uniform ergodicity and existence of the mean field under assumptions A.1--2 After that, you state assumptions A.4 and A.5 (no need of having these in the appendix).}
%
%\noteAM{I've rearranged the whole section and kept your comments just for your own reference.}

\section{Experiments}\label{sec:experiments}
In this section we illustrate numerically the performance of {\OVSMC}. We illustrate the method's capability of learning online good amortized proposals and model parameters and more complex generative models, but also show that it can serve as a strong competitor to {\VSMC} in batch scenarios. Stochastic gradients are passed to the {\ADAM} optimizer \citep{kingma:ba:2015} in Tensorflow~2\footnote{The Python code may be found at \url{https://bitbucket.org/amastrot/ovsmc}.}. All the experiments are run on an Apple MacBook Pro M1 2020, memory 8GB.

\subsection{Linear Gaussian SSM}\label{subsec:lg}
Our first example is a standard linear Gaussian SSM 
with %in which the latent process and partial observations are in 
$\set{X} = \rset^{d_x}$ and $\set{Y} = \rset^{d_y}$ for some $(d_x, d_y)\in \nsetpos^2$. More precisely, we let $\hiddens{\thp}(\cdot \mid x)$ and $\emdens{\thp}(\cdot \mid x)$ be $\Norm_{d_x}(Ax, S_u^\intercal S_u)$ and $\Norm_{d_y}(Bx, S_v^\intercal S_v)$, respectively, where $A$ and $S_u$ are $d_x \times d_x$ matrices, $B \in\rset^{d_y \times d_x}$ and $S_v \in \rset^{d_y \times d_y}$. We start with the simplest case $d_x=d_y=1$ and generate data under $A=0.8$, $B=1$ and $S_u=0.5$. We consider two cases for $S_v \in \{0.2, 1.2\}$ corresponding to informative and more non-informative observations, respectively. The state process is initialized with its stationary distribution $\Norm(0, S_u^2/(1-A^2))$. Our aim is to estimate $A$ and $S_u$ and, in parallel, optimizing the particle proposal. For this purpose, we let 
%The , in the sense that they avoid particles dispersion and the resulting computational waste. We model the latter as Gaussian distributions whose mean and variance are generated by neural networks with one dense hidden layer. In particular 
$\propdens{\thprop}(\cdot\mid x_t,y_{t+1})$ be $\Norm(\mu_\thprop(x_t,y_{t+1}),\sigma_\thprop^2(x_t,y_{t+1}))$, where $\mu_\thprop$ and $\sigma_\thprop^2$ are two distinct neural networks with one dense hidden layer having three and two nodes, respectively, and relu activation functions. In Figure~\ref{fig:lg_params} we observe the convergence of the unknown parameters for a set of distinct starting values, noticing that more informative observations yields, as expected, faster convergence. Even though values close to the true parameters are obtained already after about $10000$ to 20000 iterations, the algorithm is kept running until $t=50000$ for the purpose of learning the proposal.
\begin{figure}[htb]
	\centering
	%	\begin{subfigure}[b]{0.49\columnwidth}
		%		\centering
		\includegraphics[width=\foraistats{0.49}\forarxiv{0.4}\columnwidth]{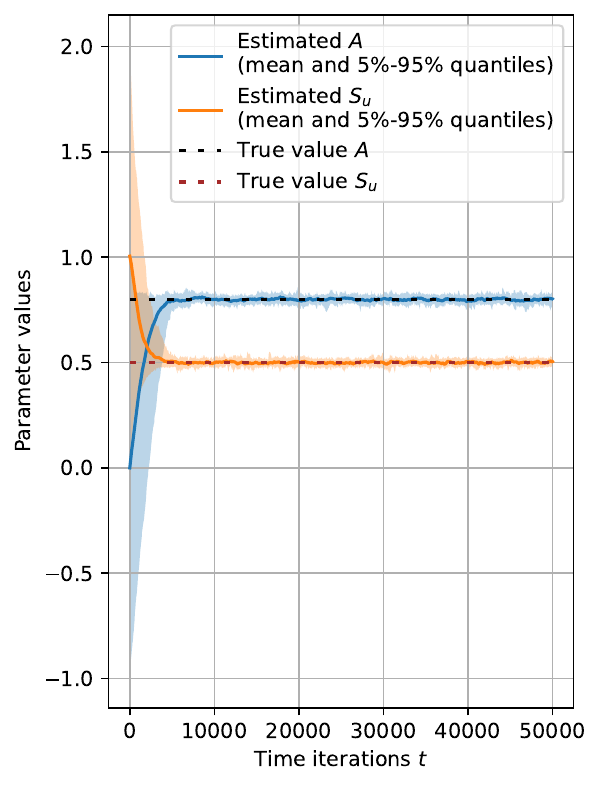}
		%		\caption{$S_v=0.2$}
		%	\end{subfigure}
	%	\hfill
	%	\begin{subfigure}[b]{0.49\columnwidth}
		%		\centering
		\includegraphics[width=\foraistats{0.49}\forarxiv{0.4}\columnwidth]{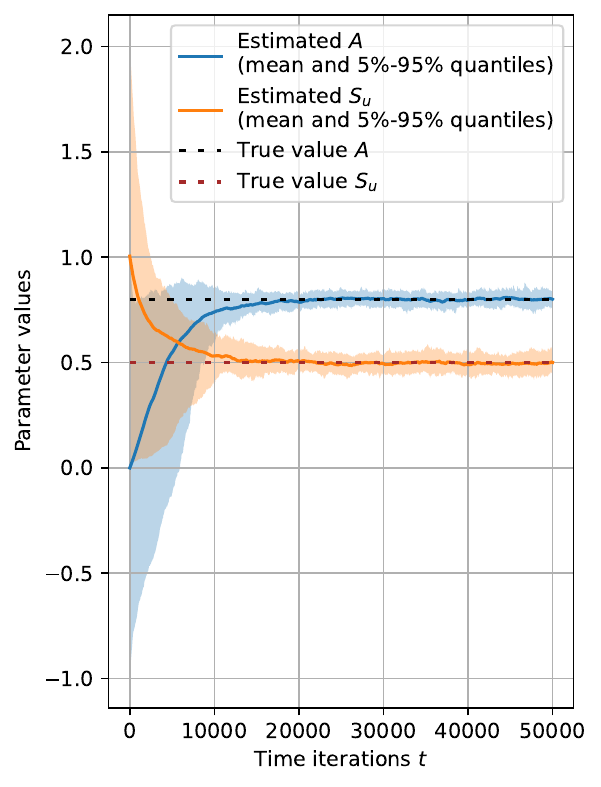}
		%		\caption{$S_v=1.2$}
		%	\end{subfigure}
	\caption{Parameter learning curves for the one-dimensional linear Gaussian SSM in Section~\ref{subsec:lg}, obtained using algorithm~\ref{algo:ovpf} with $L=5$ and $N=10000$ for $S_v=0.2$ (left) and $S_v=1.2$ (right). The means and the quantiles are calculated on the basis of 100 learning curves, each  starting with a different initial value and based on independently generated observation data.
		%, for two different observation noises.
	}\label{fig:lg_params}
\end{figure}
%Figure~\ref{fig:lg_essall} shows how the effective sample size (ESS) of the particle cloud improves while $\propdens{\thprop}$ is being learned, to finally reach the performance of the optimal proposal. This is evident for the smallest $S_v$; on the other hand, when $S_v$ is larger, the particles are well propagated into regions of non-negligible likelihood even with the bootstrap proposal, resulting in the normalized ESS being close to one regardless. 
%\begin{figure}[htb]
%	\centering
%	%	\begin{subfigure}[b]{0.49\columnwidth}
%		%		\centering
%		\includegraphics[width=\foraistats{0.49}\forarxiv{0.8}\columnwidth]{figures/lg_essall_smallvar.pdf}
%		%		\caption{$S_v=0.2$}\label{fig:lg_essall_smallvar}
%		%	\end{subfigure}
%	%	\begin{subfigure}[b]{0.49\columnwidth}
%		%		\centering
%		\includegraphics[width=\foraistats{0.49}\forarxiv{0.8}\columnwidth]{figures/lg_essall_largevar.pdf}
%		%		\caption{$S_v=1.2$}\label{fig:lg_essall_largevar}
%		%	\end{subfigure}
%	\caption{Evolution of (every $100^{\tiny{\mbox{th}}}$) normalized ESS for the one-dimensional linear Gaussian SSM in Section~\ref{subsec:lg},  for $S_v=0.2$ (left) and $S_v=1.2$ (right).}\label{fig:lg_essall}
%\end{figure}
We observe an improvement of the effective sample size (ESS) of the particle cloud as $\propdens{\thprop}$ is converging towards the optimal proposal, as shown in Figure~\ref{fig:lg_essall} in Appendix~\ref{sec:ess}. Figure~\ref{fig:lg_prop} displays
%This is noticeable in Figure~\ref{fig:lg_prop}, which display 
conditional proposal densities for some given $(x_t, y_{t+1})$ in a single run; we see that for both choices of $S_v$, the learned proposal is generally close to the locally optimal one, except in the presence of unlikely outliers in the data, in which case the learned proposal tend to move, as is desirable, its mode towards that of the prior (bootstrap) kernel. 
%\noteJO{To discuss.} \noteAM{I agree after the discussion.}
%the neural networks learn their parameters obtaining $\propdens{\thprop}(x_{t+1}\mid x_t,y_{t+1})\approx\hiddens{\thp}(x_{t+1}\mid x_t)\emdens{\thp}(y_{t+1}\mid x_{t+1})/\int \hiddens{\thp}(x_{t+1}\mid x_t)\emdens{\thp}(y_{t+1}\mid x_{t+1})\,dx_{t+1}$. Moreover, the cases that seem to behave worse correspond to values that are very unlikely to be observed in the generated data, \ie, extrapolations; nevertheless, these are still capable of showing better proposals than the bootstrap.
\begin{figure}[htb]
	\centering
%	\begin{subfigure}[b]{\columnwidth}
%		\centering
		\includegraphics[width=\forarxiv{0.8}\columnwidth]{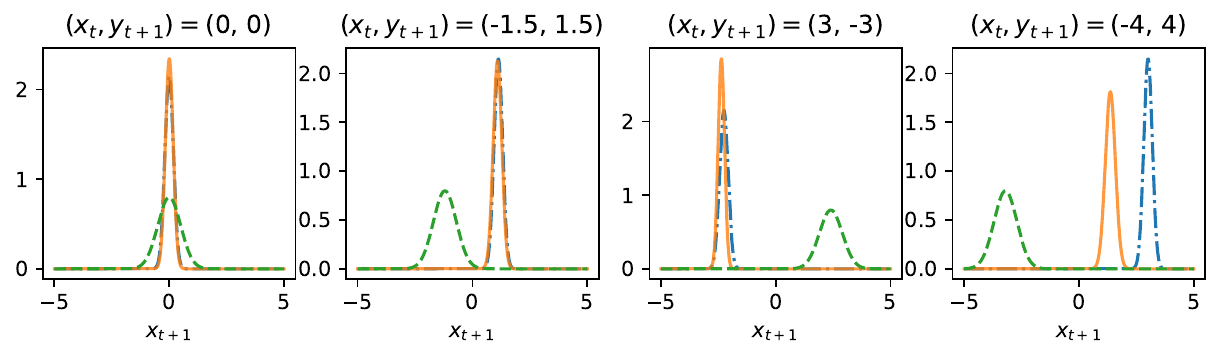}
%		\caption{$S_v=0.2$}\label{fig:lg_prop_smallvar}
%	\end{subfigure}
%	\vfill
%	\begin{subfigure}[b]{\columnwidth}
%		\centering
		\includegraphics[width=\forarxiv{0.8}\columnwidth]{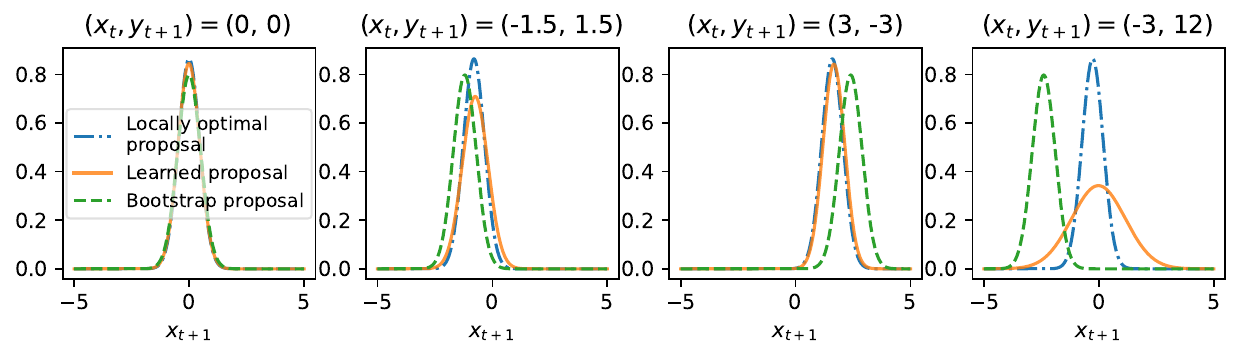}
%		\caption{$S_v=1.2$ (top densities) and $S_v=1.2$ (bottom)}\label{fig:lg_prop_largevar}
%	\end{subfigure}
	\caption{Comparisons of the bootstrap, locally optimal and learned proposals for different $(x_t, y_{t+1})$. Here {\OVSMC} was run for 50000 iterations with $L=5$ and $N=10000$  for $S_v=0.2$ (top) and $S_v=1.2$ (bottom).}\label{fig:lg_prop}
\end{figure}

Next, we show that our method can also be usefully applied in the batch-data mode. 
%Like \cite{naesseth:2018}, we focus on higher dimensions of the linear Gaussian SSM and, fixing the model parameters, observe how fast the ELBO is maximized by improving the proposals. 
Following \citet{naesseth:2018} and \citet{zhao:2021}, 
%We 
let $d_x=d_y=10$ and let the matrix $A$ have elements $A_{ij}=0.42^{\lvert i-j\rvert + 1}$. Furthermore, let $S_u=I$ and $S_v=0.5I$. We consider two parameterizations of $B$, one sparse, $B = I$, and one dense where its elements are random, independent and standard normally distributed. 
 %is randomly generated such that $(B_{ij})_{i,j}$ are i.i.d. standard normal realizations. 
 With these parameterizations, a data record $y_{0:T}$, $T = 100$, was generated starting from an $\Norm(0,I)$-distributed initial state. 
%Then we fix $T=25$ and generate data accordingly, with the initial latent state being $\Norm(0,I)$-distributed.
In this part, we view the model parameters $\thp$ as being known and focus on learning of the proposal. Instead of learning, as done by \citet{naesseth:2018} and \citet{zhao:2021}, different local proposals for each time step, we learn, in both batch and online mode, an amortized Gaussian proposal where the mean vector and diagonal covariance matrix are functions of $x_t$ and $y_{t+1}$, modeled by neural networks.    
%However, instead of the family of proposals suggested in \cite{naesseth:2018,zhao:2021}, which is not a function of the observations but learns distinct parameters for each time-step, we use a Gaussian proposal with two neural networks to model mean vector and diagonal covariance matrix as functions of $x_t$ and $y_{t+1}$. 
 These neural networks, which are time invariant, have each a single hidden layer with $16$ nodes each and the relu activation function. 
 In this setting, we (1) processed the given observation batch $M$ times using standard {\VSMC}, with gradient-based parameter updates between every sweep of the data, and (2) compared the result with the output of {\OVSMC} when executed on a data sequence of length $(T + 1)M$ formed by repeating the given data record $M$ times.  
 %In the standard {\VSMC} case, for each batch iteration, we estimate $\gradapprox_T(\thprop)$ (we do not learn $\thp$ that is fixed) and the perform the optimization step at the end of the entire particle filter; in contrast, with our OVSMC method we apply Algorithm~\ref{algo:ovpf} (skipping the optimization of $\thp$) for $T$ time-steps. 
 The number of particles was equal to $L=5$ for both methods. 
 %and stochastic gradient ascent was performed using two equally parameterized {\ADAM} optimizers. 
% ., for comparability, and we reprocess the whole data for the same amount of iterations. 
 Figure~\ref{fig:lgmult} displays the resulting ELBO evolutions for the two methods and some different {\ADAM} learning rates, and it is clear from this plot that {\OVSMC}, with its appealing convergence properties and reduced noise, is indeed a challenger of ${\VSMC}$ in this batch context. 
% \noteJO{I leave to the reader to draw his own conclusions ;)}%In this specific case our method converges much faster if the two {\ADAM} optimizers have the same learning rate, and, if we increase it for the batch {\VSMC}, we still observe a comparable behavior.
 % reached by the optimizer.    
 
% When increasing the learning rate of {\VSMC} in order to reach convergence similar to {\OVSMC}, a 
% 
% We have considered increasing the learning rate for {\VSMC}, but even if we are able to reach a similar speed of convergence with , then we pay the cost of noisier oscillations when stability is reached.

\begin{figure}[htb]
	\centering
%	\begin{subfigure}[b]{0.49\columnwidth}
%		\centering
		\includegraphics[width=\forarxiv{0.8}\columnwidth]{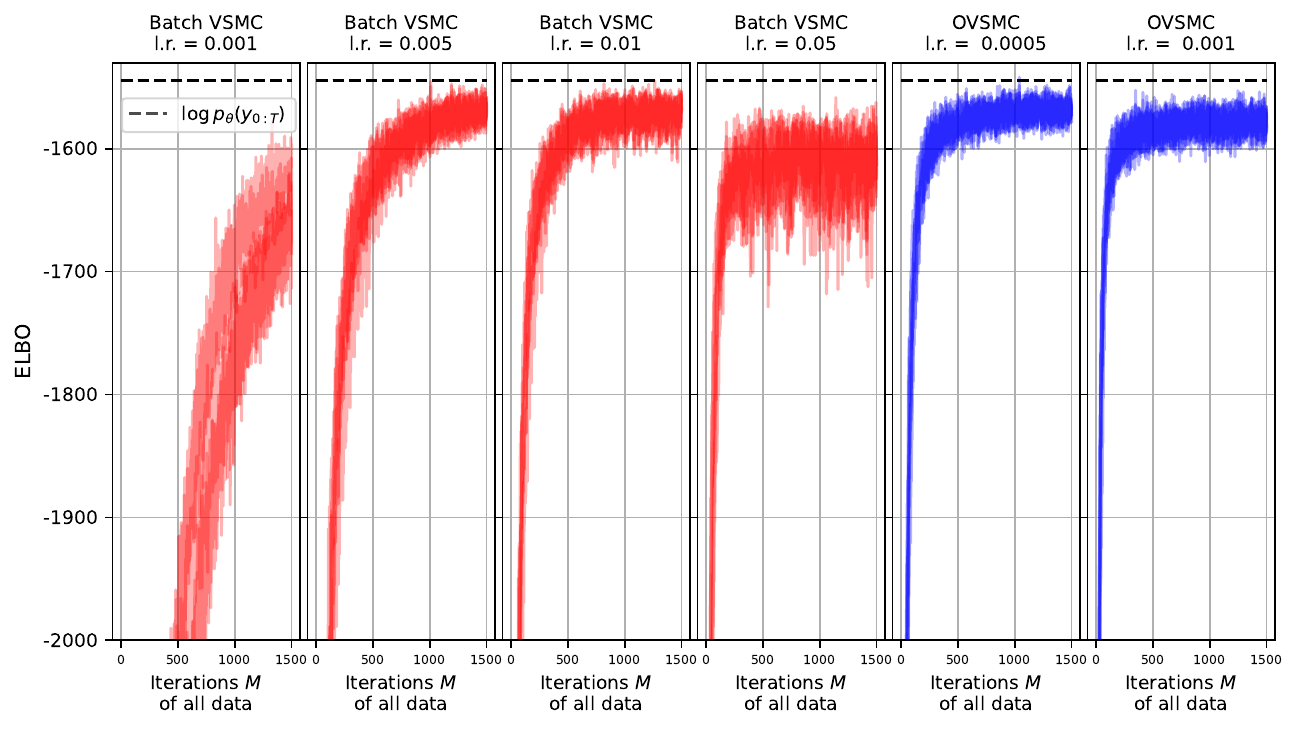}
%		\caption{Sparse $B$}
%	\end{subfigure}
%	\hfill
%	\begin{subfigure}[b]{0.49\columnwidth}
%		\centering
		\includegraphics[width=\forarxiv{0.8}\columnwidth]{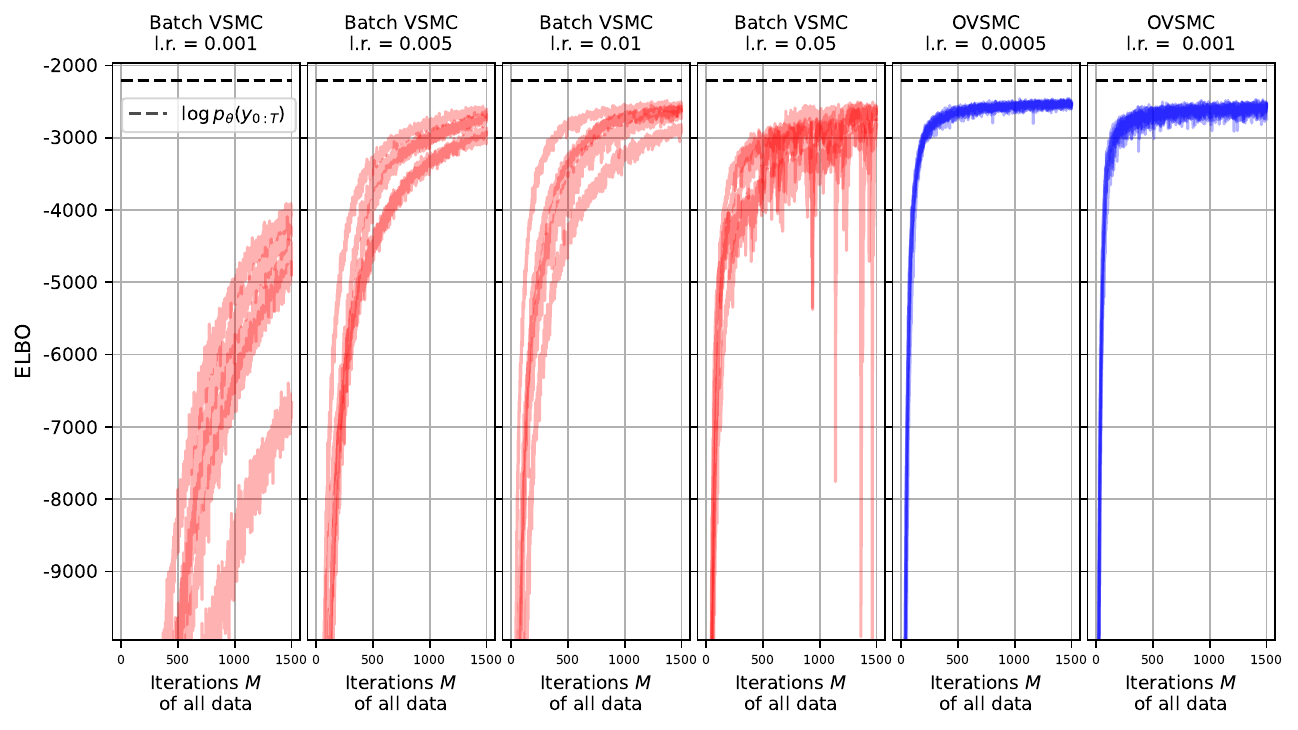}
%		\caption{Dense $B$}
%	\end{subfigure}
	\caption{ELBO evolutions of {\VSMC} and {\OVSMC} (each running with $L=5$) for the multivariate linear Gaussian SSM in Section~\ref{subsec:lg} (top: $B$ sparse; bottom: $B$ dense). 
    In each plot, which corrsponds to a particular learning rate, five independent runs och each algorithm are displayed on top of each other and compared with the target log-likelihood. 
    %\noteAM{In each plot, five independent runs of the corresponding algorithm and learning rate are displayed, superimposed on top of each other, and compared to the target log-likelihood.}
	}\label{fig:lgmult}
\end{figure}

	%. , using algorithm~\ref{algo:ovpf} with $L=5$ for two different choices of $B$. 
		%The results are compared with the fluctuations of bootstrap and optimal proposals, both with $5$ particles for comparability.

\subsubsection*{Comparison to \citet{campbell:2021}}
	As anticipated in Section~\ref{sec:intro}, the \emph{online variational filtering} ({\OVF}) approach of \citet{campbell:2021} performs online parameter learning, although considering a variational family that is not particle-based and thus not directly comparable to the {\VSMC} one. Even if it is algorithmically and computationally complex and includes a large number of hyperparameters to be tuned, {\OVF} is a relevant recent work in the context of online variational learning.  Thus, we used it as a benchmark for {\OVSMC}, and let the latter solve the same problem as in Figure~1b in \citet[Section~5.1]{campbell:2021}. The result is displayed in Figure~\ref{fig:lgcampbell}, where we demonstrate the comparability between the two methods for learning $A$ and $B$; however, we note a significant difference in the computational time, since using the code provided by \citet{campbell:2021} {\OVF} takes about 17 hours while {\OVSMC} uses less than one hour for a single run. Even though we are aware that part of this difference may be due to the particular implementation used, the complexity gap between the two methods is definitely non-negligible.
\begin{figure}[h]
	\centering
	\includegraphics[width=\columnwidth]{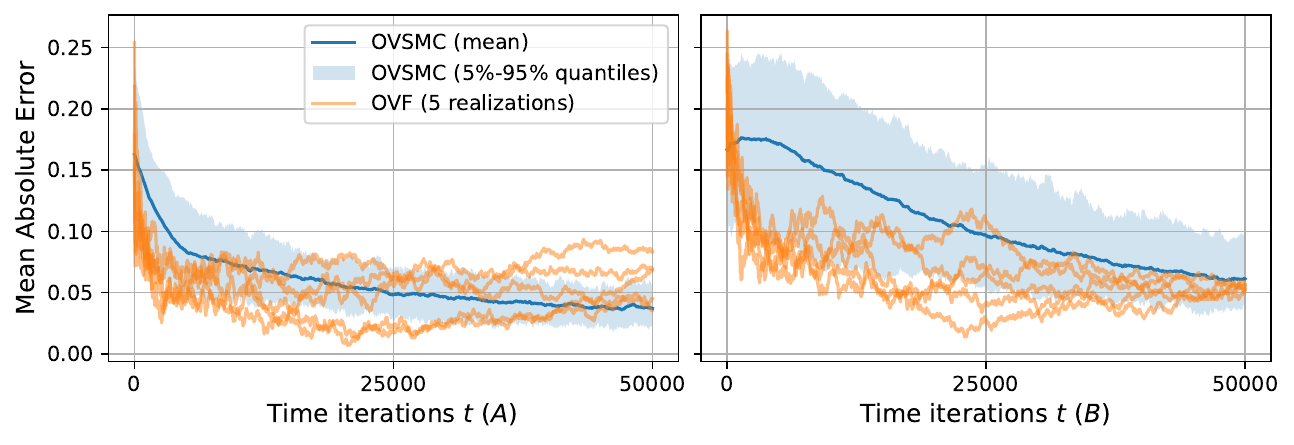}
	\caption{Mean absolute errors estimating $A$ (left) and $B$ (right) of five independent runs of {\OVF}, along with the distribution of 40 independent runs of {\OVSMC}, with $L=5$ and $N=10^4$, proposal kernel as described in Section~\ref{subsec:lg}, with 64 nodes in the hidden layer of each neural network, and {\ADAM} learning rates $10^{-3}$. Here $d_x=d_y=10$, $S_u=0.1I$, and $S_v=0.25I$ and matrices $A$ and $B$ are diagonal with i.i.d.  $\operatorname{Unif}(0.5,1)$-distributed elements. }\label{fig:lgcampbell}
\end{figure}

\subsection{Stochastic Volatility and RML}\label{subsec:sv_rml}
In this example we focus on parameter estimation and compare {\OVSMC} with \texttt{PaRIS}-based RML \citep{olsson:westerborn:2018}. We consider a univariate \textit{stochastic volatility model} \citep{hull:white:1987}, where for $x\in \set{X} = \rset$, $\hiddens{\thp}(\cdot\mid x)$ and $\emdens{\thp}(\cdot\mid x)$ are $\Norm(\alpha x, \sigma^2)$ and $\Norm(0, \beta^2 \exp(x))$, respectively, with $\alpha\in(0,1)$, $\sigma > 0$ and $\beta > 0$. The state process is initialized according to its $\Norm(0,\sigma^2/(1-\alpha^2))$ stationary distribution. Figure~\ref{fig:sv_rml_conv} reports the learning curves of the estimated parameters $\thp = (\alpha, \sigma, \beta)$ obtained with {\OVSMC} and particle-based RML, \noteJO{starting from different parameter vectors}, for 27 independent data sequences
%\noteAM{starting from different combinations of initial points, for 27 independent data sequences} 
generated under $\theta=(0.975,0.165,0.641)$. In both algorithms, the learning rate was $10^{-3}$.
\begin{figure}[h]
	\centering
	\includegraphics[width=\forarxiv{0.8}\columnwidth]{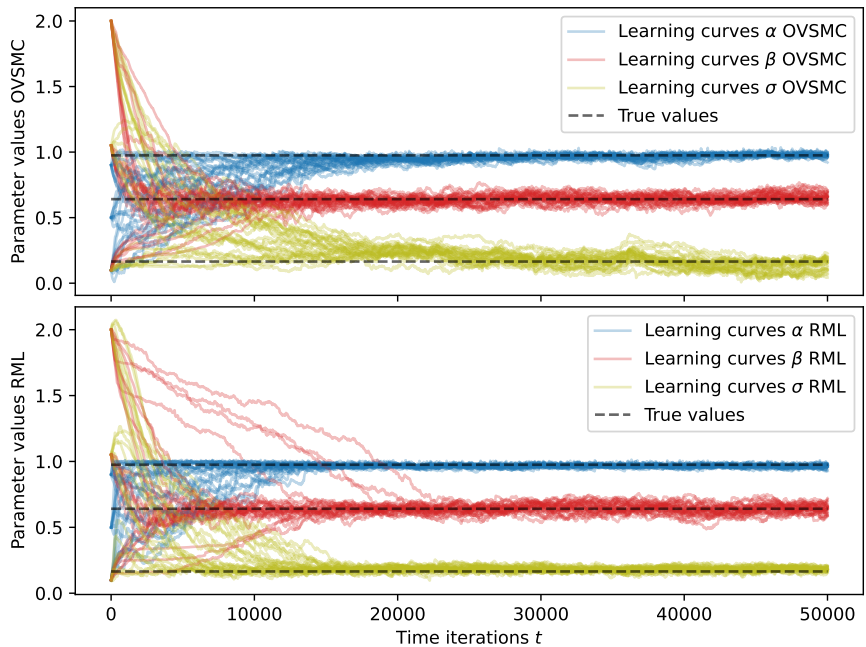}
	\caption{Parameter learning curves obtained with {\OVSMC} (with $L = 5$ and $N = 1000$) and \texttt{PaRIS}-based RML  \citep{olsson:westerborn:2018} (with $N = 1000$), for the stochastic volatility SSM in Section~\ref{subsec:sv_rml}, with learning rate $10^{-3}$.}\label{fig:sv_rml_conv}
\end{figure}
\noteJO{The learning curves show similar behavior, with {\OVSMC} converging faster than RML in several cases when estimating $\beta$, but fluctuating somewhat more when learning $\alpha$ and $\sigma$. This is to be expected, since the RML procedure is based on particle smoothing. However, we should remember that {\OVSMC} does not use any backward sampling and in addition to learning $\thp$ also has the ability to optimize online the proposal kernel over the same family of deep Gaussian proposals as in Section~\ref{subsec:lg}.}
%\noteAM{We observe comparable behaviors in the learning curves, where in several cases {\OVSMC} converges faster than RML in the estimation of $\beta$, but it fluctuates more in the learning of $\alpha$ and $\sigma$, which is somehow expected, as the RML procedure is based on particle smoothing. Furthermore, it is crucial to remark that {\OVSMC} is not only learning $\thp$, but also optimizes online the proposal over the same family of deep Gaussian proposals as in Section~\ref{subsec:lg}.} 
As the number of particles grows, the computational speed is also in favor of {\OVSMC}, since the \texttt{PaRIS}-based RML has a computational bottleneck stemming from its inherent backward-sampling mechanism; with our implementation, {\OVSMC} was three times faster than \texttt{PaRIS}-based RML in this specific case, and this ratio grows even more in favor of {\OVSMC} as $N$ increases.

\subsection{Deep Generative Model of Moving Agent}\label{subsec:ball}
In this section we study the applicability of {\OVSMC} to more complex and high-dimensional models. Inspired by \citet[Section~5.3]{le:2018}, we produced a long and partially observable video sequence, with frames represented by $32\times 32$ arrays, showing a moving agent. The motion is similar as the one described in the referenced paper, with the difference that the agent exhibits a stationary behavior by bouncing against the edges of the image. The agent is occluded from the image whenever it moves into the central region of the frame, covered by a $16\times 30$ horizontal rectangle. The data generation is described in more detail in Appendix~\ref{subsec:datagenball}. 
%\begin{figure}[h]
%	\centering
%	\includegraphics[width=\forarxiv{0.8}\columnwidth]{figures/elbo_ball.pdf}
%	\caption{The deep generative moving-agent model in Section~\ref{subsec:ball}. Moving averages with window size 500 of $\log(\sum_{i=1}^{L}\wgtfunc{t}{i}(\thprop,\thp)/L)_{t = 0}^T$ (proposal learning) and $\log(\sum_{i=1}^{N}\wgtfunc{t}{i}(\thprop,\thp)/N)_{t = 0}^T$ (model learning), $T = 10^5$, produced using {\OVSMC} with $L=5$ and $N=20$. The same quantities are shown for two test videos, with $N=20$ and $N=500$, generated by the learned model (and no learning is performed).}\label{fig:elbo_ball}
%\end{figure}
Following \citet[Section~C.1]{le:2018}, we model the generative process and the proposal through the framework of \textit{variational recurrent neural networks} \citep{chung:2015}, with a similar architecture described in more detail in Appendix~\ref{subsec:vrnnball}. 

In this setting, we generated a single video sequence with $T = 10^5$ frames, and processed the same by iterating Algorithm~\ref{algo:ovpf}, with $N=20$ and $L=5$, $T$ times. We remark that since we have to run the algorithm for a large number of iterations, it is clearly infeasible---and possibly not even useful---to input the whole history of frames and latent states to the recurrent neural network of the model; instead, we include only a fixed window (of length $40$) of recent frames and states and discard all the previous information. This allows the average time and memory requirement per iteration to be kept at a constant value. 
%Figure~\ref{fig:elbo_ball} displays the moving average of the incremental ELBOs $(\log( \wgtsum{t}/N))_{t \in \nset}$, where we have chosen a moving-average window of size $500$ to partially smooth both the random fluctuations and the periodical patterns of the video, and we see how this quantity increases with the number of iterations. In addition, we provide the same type of plot for two additional videos, generated with different seeds, where no learning is performed but only the incremental ELBOs are evaluated with the previously learned model. A visual illustration of the method's capability to generate likely videos is presented in Appendix~\ref{subsec:ballsampling}.

\noteJO{Figures~\ref{fig:video2}--\ref{fig:video4} show a visual illustration of the method's ability to generate probable videos. The goal is to analyze the quality of the learned generative model by visualizing some independently sampled videos and comparing them to a reference video generated from a different seed than the one used for training. We consider a 150 frames long reference video. For each sample we use the proposal distribution to generate the latent states for the first 60 iterations, \ie, we access the true frame to sample and then decode the states, and after $t=60$ we continue with the generative model only. The images show every third frame for the true video and five independent sample videos. We note that the first 60 frames of the sample videos look similar to the true video, as expected since the model has access to the true observations. 
After that, the sample videos continue similarly in an initial phase, but start to move at different speeds when the agent disappears behind the rectangle after $t=75$.
As the generated frame sequences show, the model predicts a smooth movement pattern of the agent, with some random perturbations, as in the second sample, where the agent leaves the rectangle after $t=96$ from the top instead of below.}

\begin{figure}[htb]
	\centering
	\includegraphics[width=\forarxiv{0.8}\columnwidth]{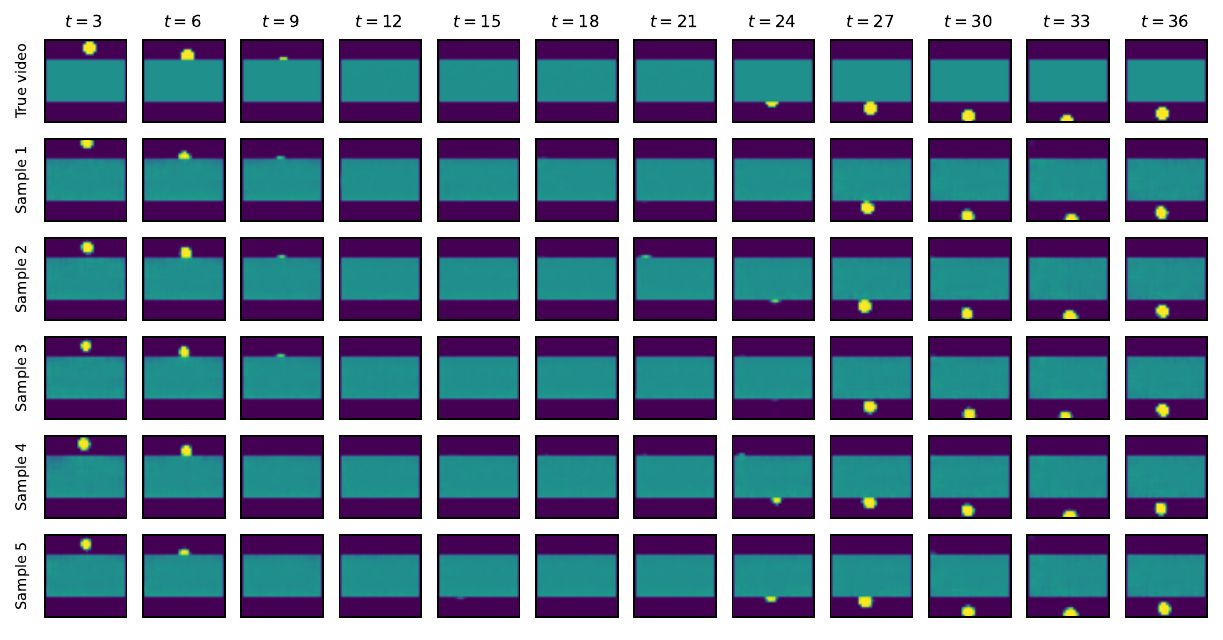}
%	\caption{Comparison between original video and five samples under the learned model. Until $t=60$, frames are generated using the learned proposal and the given data to encode the latent states; after this, the frames are produced using the generative model. Only every $3$rd frame is displayed.}\label{fig:video1}
%\end{figure}
%\begin{figure}[h]
%	\centering
	\includegraphics[width=\forarxiv{0.8}\columnwidth]{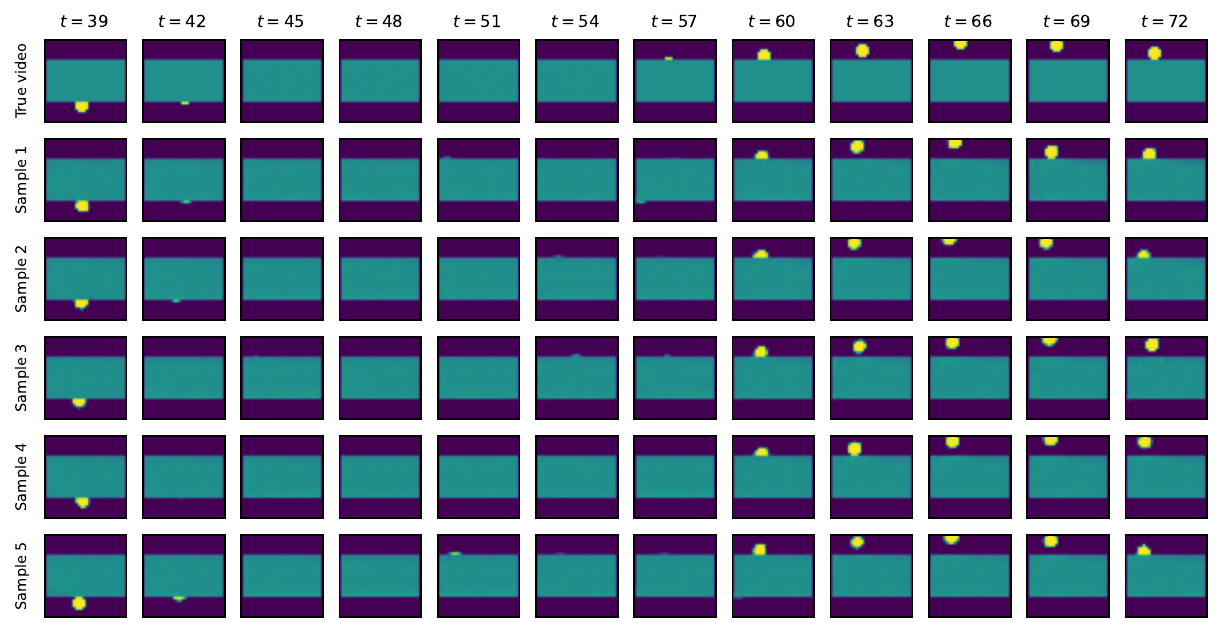}
	\caption{See caption in Figure~\ref{fig:video4}.}\label{fig:video2}
%	\caption{Comparison between original video and five samples under the learned model. Until $t=60$, frames are generated using the learned proposal and the given data to encode the latent states; after this, the frames are produced using the generative model. Only every $3$rd frame is displayed.}
\end{figure}
\begin{figure}[htb]
%	\centering
	\includegraphics[width=\forarxiv{0.8}\columnwidth]{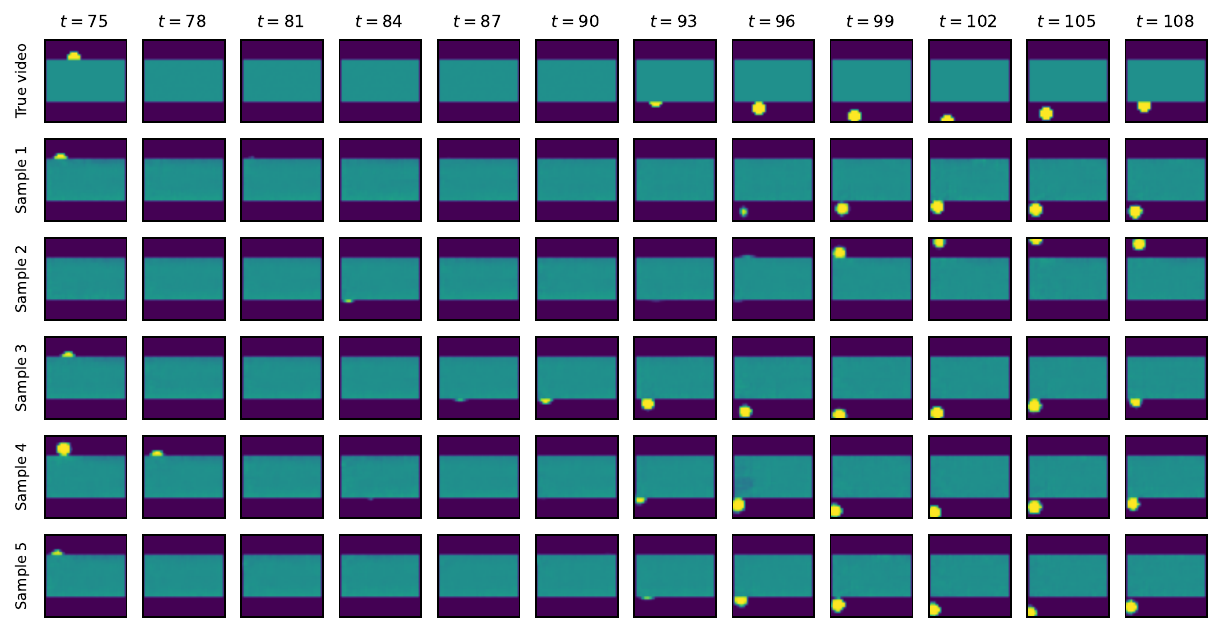}
%	\caption{Comparison between original video and five samples under the learned model. Until $t=60$, frames are generated using the learned proposal and the given data to encode the latent states; after this, the frames are produced using the generative model. Only every $3$rd frame is displayed.}\label{fig:video3}
%\end{figure}
%\begin{figure}[h]
%	\centering
	\includegraphics[width=\forarxiv{0.8}\columnwidth]{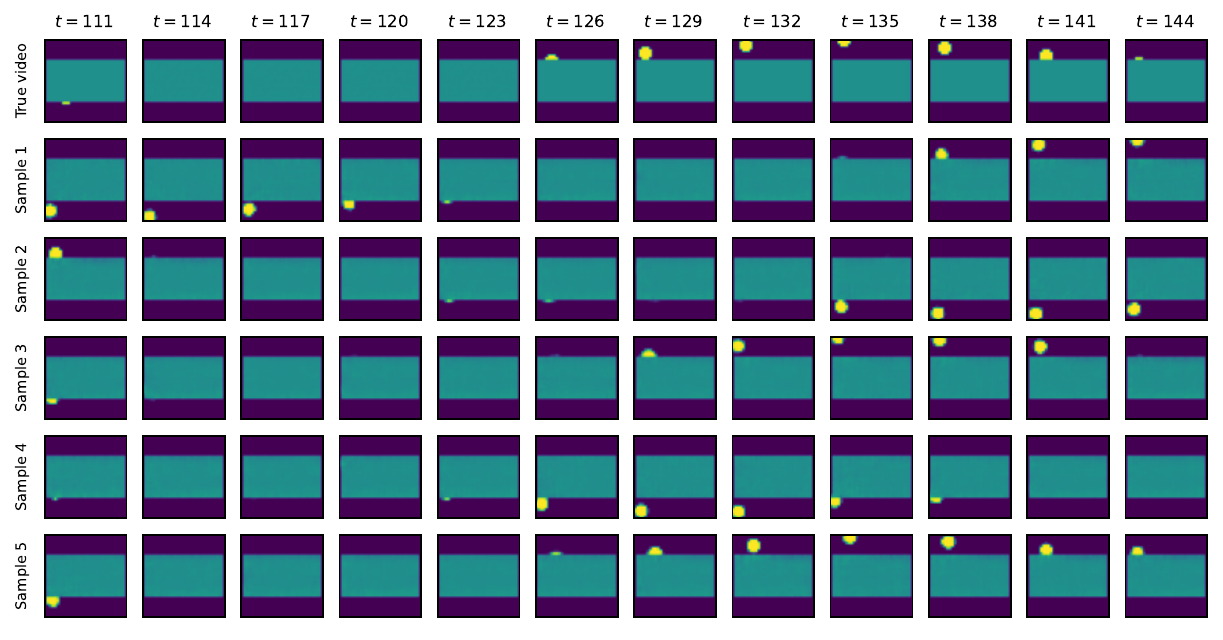}
	\caption{Comparison between original video and five samples under the learned model. Until $t=60$, frames are generated using the learned proposal and the given data to encode the latent states; after this, the frames are produced using the generative model. %Only every $3$rd frame is displayed.
 }\label{fig:video4}
\end{figure}

\section{Conclusions}
We have presented 
%the \textit{online variational sequential Monte Carlo} (\OVSMC)
the {\OVSMC} algorithm, a procedure that extends the batch {\VSMC} to the context of streaming data. The proposed methodology allows us to learn, on-the-fly, unknown model parameters and optimal particle proposals in both standard SSM as well as and more complex generative models. Under strong mixing assumptions, which are standard in the literature, we provide theoretical support showing that the stochastic approximation scheme of {\OVSMC} solves the same problem as {\VSMC} for an infinite batch size. The interesting question of whether the estimates produced by {\OVSMC} are still consistent for the true parameter in the case of correctly specified SSM despite the biased approximation of the ELBO gradient remains open for future work. 
\noteJO{Moreover, from an application point of view, it would be appealing to explore, empirically and theoretically, the case where the parameters vary through time (\ie when dealing with non-stationary data), with the aim of identifying the conditions under which {\OVSMC} would still work.}
%\noteAM{Moreover, from an application point of view, it would be appealing to explore, empirically and possibly theoretically, the case where the parameters vary through time (\ie when dealing with non-stationary data), with the aim of identifying the conditions under which {\OVSMC} would still work.}

\newpage
%\bibliography{motherofallbibs}
%\bibliographystyle{apalike}

\newpage
\appendix

\section{Proof of Theorem~\ref{thm:main}}\label{appendix:a}
\newcommand{\refg}{\eta}

In the first part of this supplement we provide a detailed proof of Theorem~\ref{thm:main}, proceeding as follows.  
\begin{itemize}
	\item\textbf{Preliminaries (Section~\ref{subsec:prel}).} For technical reasons, the proof of Theorem~\ref{thm:main} to be presented calls for some more sophisticated notation, introduced in Section~\ref{sec:notation}. In addition, in Section~\ref{sec:algorithm} we present the slight modification of Algorithm~\ref{algo:ovpf} that is the subject of our theoretical analysis and in which we optimize a unique parameter vector $\thp$ containing both model and proposal parameters.
	\item \textbf{Intermediate results (Section~\ref{subsec:int_res}).} In this part, we redefine the extended state-dependent Markov chain $(Z_t)_{t \in \nset}$ (discussed in Section~\ref{sec:theory}), including the original SSM as well as the random variables generated by Algorithm~\ref{algo:ovpf_theory}. Under strong mixing assumptions, we establish, in Proposition~\ref{prop:ergod} (corresponding to Proposition~\ref{prop:ergomain}) uniform ergodicity of the Markov transition kernel $T_\thp$ governing $(Z_t)_{t \in \nset}$ and the existence of a stationary distribution. 
	
	After this, we consider the Robbins--Monroe scheme with state-dependent Markov noise targeting the mean field $\mf{\thp}$ defined (as in Section~\ref{sec:theory}) as the expectation of the noisy measurement function $\grad{\thp}$ under the stationary distribution of $T_\thp$, and prove that it is bounded. Under the assumption that the state-process and emission transition densities as well as their gradients (with respect to $\thp$) are bounded and Lipschitz continuous in $\thp$, Proposition~\ref{prop:poisson} then establishes the existence of a solution to the Poisson equation associated with the Markov kernel $T_\thp$.
	\item \textbf{Main proofs (Section~\ref{subsec:main_proofs}).} Next, we introduce further assumptions regarding the Lyapunov function associated with the mean field and the step-size sequence of the stochastic update, under which we restate and prove Theorem~\ref{thm:main} and Corollary~\ref{corollary:decay}.
	\item \textbf{Auxiliary results (Section~\ref{subsec:otherproofs}).} Finally, we prove Proposition~\ref{prop:poisson} using Lemmas~\ref{lemma:boundswgt} and \ref{lemma:1}. The approach to these proofs is partially inspired by the work of \citet{tadic:doucet:2018}.
\end{itemize}
We remark that our analysis is carried through for a fixed particle sample size $N \in \nsetpos$ in the SMC algorithm, and it is beyond the scope of our paper to optimize the derived theoretical bounds with respect to $N$. 

\subsection{Preliminaries}\label{subsec:prel}
\subsubsection{Notation}
\label{sec:notation}
We let $\rsetnonneg$ and $\rsetpos$ be the sets of nonnegative and positive real numbers, respectively, and denote vectors by $x_{m:n}\eqdef(x_m,x_{m+1},\dots,x_{n-1},x_n)$ or, alternatively, by $x^{m:n}\eqdef(x^m,x^{m+1},\dots,x^{n-1},x^n)$, depending on the situation. For some general state space $(\set{S}, \alg{S})$, we let $\meas{\alg{S}}$ be the set of measures on $\alg{S}$, and $\probmeas{\alg{S}}\subset\meas{\alg{S}}$ the probability measures.  For a signed measure $\sigma(ds)$ on $(\set{S},\alg{S})$, we denote by $\abs{\sigma}(ds)$ its total variation. 

We now reintroduce the SSM $(X_t, Y_t)_{t \in \mathbb{N}}$, specified in Section~\ref{sec:background}, in a somewhat more rigorous way. More specifically, $(X_t, Y_t)_{t \in \mathbb{N}}$ is defined as a 
bivariate Markov chain evolving on $(\set{X} \times \set{Y}, \alg{X} \tensprod \alg{Y})$ according to the Markov transition kernel 
$$
\kernel{K}_\thp : (\set{X}\times\set{Y}) \times (\alg{X} \tensprod \alg{Y}) \ni ((x,y), A) \mapsto \iint \1{A}(x',y') \, \hidker{\thp}(x,dx') \,\emker{\thp}(x', dy'), 
$$
where 
\begin{align}
\hidker{\thp} : \set{X} \times \alg{X} &\ni (x, A) \mapsto \int  \1{A}(x') \, \hiddens{\thp}(x, x') \, \mu(dx'), \\
\emker{\thp} : \set{X} \times \alg{Y} &\ni (x, B) \mapsto \int  \1{B}(y) \, \emdens{\thp}(x, y) \, \refg(dy),
\end{align}
with $\hiddens{\thp}: \set{X} \times \set{X} \to \rsetnonneg$ and $\emdens{\thp}:\set{X}\times\set{Y}\to\rsetnonneg$ being the state and emission transition densities and $\mu \in \meas{\alg{X}}$ and $\refg \in \meas{\alg{Y}}$ reference measures. (Here we have slightly modified the notation in Section~\ref{sec:background}, by using the short-hand notation $\hiddens{\thp}(x, x') = \hiddens{\thp}(x' \mid x)$ and $\emdens{\thp}(x, y) = \emdens{\thp}(y \mid x)$.)
The chain is initialized according to $\xinit\varotimes\emker{\thp}:\alg{X} \tensprod \alg{Y} \ni A \mapsto \int_A \xinit(dx) \,\emker{\thp}(x, dy)$, where $\xinit$ is some probability measure on $(\set{X}, \alg{X})$ having density $m_0(x)$ with respect to $\mu$. As specified in Section~\ref{sec:background}, only the process $(Y_t)_{t \in \nset}$ is observed, whereas the state process $(X_t)_{t \in \nset}$ is unobserved and hence referred to as \emph{latent} or \emph{hidden}. 
Moreover, $\thp$ is a parameter belonging to some vector space $\parspace$ and governing the dynamics of the model. 

\subsubsection{Algorithm}
\label{sec:algorithm}

We let $\propker{\thp}: \set{X} \times \set{Y} \times \alg{X} \to [0,1]$ be some Markov kernel, the so-called \emph{proposal kernel}, parameterized by $\thp \in \parspace$ as well and having transition density $\propdens{\thp}:\set{X}\times\set{X}\times\set{Y}\to \rsetnonneg$ with respect to $\mu$, such that for every $(x, y, A) \in \set{X} \times \set{Y} \times \alg{X}$, 
$$
\propker{\thp}((x, y), A) = 0 \Rightarrow \int \1{A}(x') \emdens{\thp}(x', y) \, \hidker{\thp}(x, dx') = 0. 
$$ 
On the basis of the proposal kernel, define the weight function $\wgtfuncb{\thp}{}(x,x',y)\eqdef\hiddens{\thp}(x,x')\emdens{\thp}(x',y)/\propdens{\thp}(x,x',y)$ for $(x,x',y)\in\set{X}\times\set{X}\times\set{Y}$ such that $\propdens{\thp}(x, x', y) > 0$. In order to express the particles as explicit differentiable functions of $\thp$, the proposal is assumed to be reparameterizable. More precisely, we assume that there exist some state-space $(\set{E},\alg{E})$, an easily sampleable probability measure $\indmeas\in\probmeas{\alg{E}}$ not depending on $\thp$, and a function $\repfunc{\thp}: \set{X} \times \set{Y} \times \set{E}\to \set{X}$ such that for all $(x,y)\in\set{X}\times\set{Y}$ and $\thp \in \parspace$, it holds that $\int h(\repfunc{\thp}(x, y, v)) \, \indmeas(d v) = \int h(x') \, \propker{\thp} ((x,y),dx')$ for all bounded real-valued measurable functions $h$ on $\set{X}$; in other words, the pushforward distribution $\indmeas \circ \repfunc{\thp}^{-1}(x, y, \cdot)$ coincides with $\propker{\thp}((x, y), \cdot)$. 

Algorithm~\ref{algo:ovpf_theory} displays the procedure studied in our analysis. In this slightly modified version of Algorithm~\ref{algo:ovpf}, the particle system will be represented by the \emph{resampled} particles, denoted here as $(\epartil{t}{i})_{i=1}^N$, $t\ge 0$. In order to avoid introducing further notation, the resampled particles are conventionally all initialized at time $-1$ by some arbitrary value $u\in\set{X}$ such that  $\{\repfunc{\thp}(u,y_0,\auxrv_{0}^i)\}_{i=1}^N$, corresponding to $(\epart{0}{i})_{i=1}^N$, are i.i.d. according to some initial proposal distribution $\propker{\thp}((u,y_0), \cdot)$ depending on $u$, where $(\auxrv_{0}^i)_{i=1}^N\sim\indmeas^{\tensprod N}$. As we will see, the cloud of resampled particles will be included in the state-dependent Markov chain $(Z_t)_{t \in \nset}$ governing the perturbations of the stochastic-approximation scheme under consideration; the initialization according to the constant $u$ is described in \eqref{eq:init_dist}.

In Algorithm~\ref{algo:ovpf_theory}, we operate with \emph{two} samples of mutated particles, one used for the propagation of the particle cloud and one used for approximation of the gradient, in the sense that the noise variables $(\auxrv_{t+1}^i)_{i = 1}^N$ generated on Line~9 are not used to propagate the particles. Consequently, $\repfunc{\thp_t}(\epartil{t}{i}, y_{t+1}, \auxrv_{t+1}^i)$ is generally different from $\epart{t + 1}{i}$.   
Importantly, this decoupling makes the Markov transition kernel $T_\thp$ non-collapsed (free from Dirac components), which, as we will see, facilitates significantly the theoretical analysis of the algorithm. 
%\begin{algorithm}[h]
%	\caption{{\OVSMC} (modified version)}\label{algo:ovpf_theory}
%	\begin{algorithmic}[1]
%		\Statex \textbf{Input}: $(\epartil{t-1}{i})_{i=1}^N$, $y_{t:t+1}$, $\thp_t$
%		\For{$i=1,\dots,N$}
%		\State draw $\epart{t}{i}\sim \propker{\thp_{t}}((\epartil{t-1}{i},y_{t}),\cdot)$;
%		\State set $\wgt{t}{i}\gets\wgtfuncb{\thp_{t}}{}(\epartil{t-1}{i},\epart{t}{i},y_{t})$;
%		\EndFor
%		\For{$i=1,\dots,N$}
%		\State draw $\I{t+1}{i}\sim \catdist((\wgt{t}{\ell})_{\ell=1}^N)$;
%		\State set $\epartil{t}{i}\gets \epart{t}{\I{t+1}{i}}$;
%		\State draw $\auxrv_{t+1}^{i}\sim \indmeas$;
%		\EndFor
%		\State set $\thp_{t+1}\gets\thp_t+\gamma_{t+1}\nabla\log\left(\sum_{i=1}^{N}\wgtfuncb{\thp_t}{}(\epartil{t}{i},\repfunc{\thp_t}(\epartil{t}{i}, y_{t+1}, \auxrv_{t+1}^i),y_{t+1})\right)$;
%		\State \Return $(\epartil{t}{i})_{i=1}^N$, $\thp_{t+1}$
%	\end{algorithmic}
%\end{algorithm}
\begin{algorithm}[h]
	\caption{{\OVSMC} (modified version)}\label{algo:ovpf_theory}
	\begin{algorithmic}[1]
		\STATE {\bfseries Input:} $(\epartil{t-1}{i})_{i=1}^N$, $y_{t:t+1}$, $\thp_t$
		\FOR{$i=1,\dots,N$}
		\STATE draw $\epart{t}{i}\sim \propker{\thp_{t}}((\epartil{t-1}{i},y_{t}),\cdot)$;
		\STATE set $\wgt{t}{i}\gets\wgtfuncb{\thp_{t}}{}(\epartil{t-1}{i},\epart{t}{i},y_{t})$;
		\ENDFOR
		\FOR{$i=1,\dots,N$}
		\STATE draw $\I{t+1}{i}\sim \catdist((\wgt{t}{\ell})_{\ell=1}^N)$;
		\STATE set $\epartil{t}{i}\gets \epart{t}{\I{t+1}{i}}$;
		\STATE draw $\auxrv_{t+1}^{i}\sim \indmeas$;
		\ENDFOR
		\STATE set $\thp_{t+1}\gets\thp_t+\gamma_{t+1}\nabla\log\left(\sum_{i=1}^{N}\wgtfuncb{\thp_t}{}(\epartil{t}{i},\repfunc{\thp_t}(\epartil{t}{i}, y_{t+1}, \auxrv_{t+1}^i),y_{t+1})\right)$;
		\STATE {\bfseries return} $(\epartil{t}{i})_{i=1}^N$, $\thp_{t+1}$.
	\end{algorithmic}
\end{algorithm}
\subsection{Intermediate results}\label{subsec:int_res}
\subsubsection{Construction of $(Z_t)_{t \in \nset}$}
We first provide a more detailed statement of Assumption~\ref{assum:ssm}, which assumes that the law of the data is governed by an unspecified SSM.
\begin{assumption}\label{assum:ssmapp}
	The observed data stream $(Y_t)_{t \in \nset}$ is the output of an SSM $(X_t, Y_t)_{t \in \nset}$ on $(\set{X} \times \set{Y}, \alg{X} \tensprod \alg{Y})$ with some state and observation transition kernels $\hidkertrue(x,dx')$ and $\emkertrue(x,dy)$, respectively. These kernels have transition densities $\hiddenstrue(x,x')$ and $\emdenstrue(x,y)$ with respect to $\mu$ and $\refg$, respectively.
\end{assumption}
As explained in Section~\ref{sec:theory}, the stochastic process $(\z_t)_{t \in \nset}$, evolving on the product space $(\set{Z}, \alg{Z})\eqdef(\set{X}\times \set{Y} \times \set{X}^N \times \set{E}^N, \alg{X} \tensprod\alg{Y}\tensprod\alg{X}^{\tensprod N} \tensprod \alg{E}^{\tensprod N})$, includes the data-generating SSM well as the random variables generated by Algorithm~\ref{algo:ovpf_theory}, \ie, for $t\in\nset$, 
$\z_t \coloneqq (X_t, Y_t, \epartil{t-1}{1:N}, \auxrv_{t}^{1:N})$. Let $\prob$ and $(\mathcal{F}_t)_{t \in \nset}$ be the law and natural filtration of $(Z_t)_{t \in \nset}$; then, as described in Section~\ref{sec:theory}, $(\z_t)_{t\in\nset}$ is a state-dependent Markov chain with transition kernel $\tz{\thp}$, in the sense that for any bounded measurable function $h$ on $\set{Z}$, $\prob$-a.s., $\mathbb{E}[h(Z_{t + 1}) \mid \mathcal{F}_t] = T_{\thp_t} h (Z_t)$. The kernel $T_\thp$ is given by, with 
$z_t = (x_t, y_t, \tilde{x}_{t - 1}^{1:N}, v_t^{1:N})$, 
\begin{multline}
	\tz{\thp}(z_{t},dz_{t+1}) \eqdef \hidkertrue(x_t, dx_{t+1})\,\emkertrue(x_{t+1},dy_{t+1}) \int_{\epartb{t}{1:N}}
	\prod_{k=1}^{N}\propdens{\thp}(\epartilb{t-1}{k},\epartb{t}{k},y_t) \\ \times\prod_{j=1}^{N}\left(\sum_{i=1}^{N}\frac{\wgtfuncb{\thp}{}(\epartilb{t-1}{i},\epartb{t}{i},y_t)}{\sum_{\ell=1}^{N}\wgtfuncb{\thp}{}(\epartilb{t-1}{\ell},\epartb{t}{\ell},y_t)} \delta_{\epartb{t}{i}}(d\epartilb{t}{j}) \right)\,\indmeas^{\tensprod N}(d\auxrvb_{t+1}^{1:N})\,\mu^{\tensprod N}(d\epartb{t}{1:N}), 
\end{multline}
where we have written $\int_{\epartb{t}{1:N}}$ to indicate that the integral is with respect to $(x_t^i)_{i = 1}^N$. The initial state $\z_0^\thp$ is initialized according to the probability measure
\begin{equation}\label{eq:init_dist}
	\tau_{0}(dz_0) \eqdef \xinit_{}(dx_0)\,\bar{\emker{}}(x_{0},dy_{0})\,\delta_u^{\tensprod N}(d\epartilb{-1}{1:N}) \, \indmeas^{\tensprod N}(d\auxrvb_0^{1:N}),
\end{equation}
where the dummy particles $\epartilb{-1}{1:N}$ we are initialized at an arbitrary point $u \in \set{X}$. Under the following assumption, we will establish that the kernel $T_\thp$ is uniformly ergodic and has an invariant distribution.
\begin{assumption}\label{assum:2}
	There exists $\eq \in(0,1)$ such that for every $\thp\in \parspace$, $(x,x')\in\set{X}^2$ and $y\in\set{Y}$, 
	$$\eq\le \hiddenstrue(x,x')\le\frac{1}{\eq},\qquad\eq\le \hiddens{\thp}(x,x')\le\frac{1}{\eq},\qquad \eq\le\emdens{\thp}(x,y)\le\frac{1}{\eq},\qquad \eq\le\propdens{\thp}(x,x',y)\le\frac{1}{\eq}.
	$$ 
\end{assumption}
The strong mixing assumptions of Assumption~\ref{assum:2} (as well as Assumption~\ref{assum:3} introduced later) are standard in the literature and point to applications where the state and parameter space are compact sets. Note that from Assumption~\ref{assum:2} it follows directly that $\eq^3\le \wgtfuncb{\thp}{}(x,x',y)\le\eq^{-3}$ for all $\thp\in \parspace$, $(x,x')\in\set{X}^2$ and $y\in\set{Y}$. Let $(\z_t^\thp)_{t \in \nset}$ and $\prob_\thp$ denote the canonical Markov chain and its law induced by the Markov transition kernel $T_\thp$. 
	\begin{remark} \label{rem:mu:probability}
	Note that by Assumption~\ref{assum:2} it follows that the reference measure $\mu$ is a finite measure on $(\set{X}, \alg{X})$. Thus, without loss of generality we assume in the following that $\mu$ is a probability measure. 
	\end{remark}

We now rewrite and prove Proposition~\ref{prop:ergomain}.
\begin{proposition}[Proposition~\ref{prop:ergomain}] \label{prop:ergod}
	Let Assumptions~\ref{assum:ssmapp} and \ref{assum:2} hold. Then for every $\thp\in\parspace$, $(\z_t^\thp)_{t \in \nset}$ is geometrically uniformly ergodic and admits a unique stationary distribution $\statmeas{\thp}\in\probmeas{\alg{Z}}$.
\end{proposition}
\begin{proof}
	In order to establish uniform ergodicity we show that $T_\theta$ allows the state space $\set{Z}$ as a $\nu_1$-small set for some $\nu_1\in\meas{\alg{Z}}$. Indeed for any $z_t = (x_t, y_t, \tilde{x}_{t - 1}^{1:N}, v_t^{1:N}) \in \set{Z}$ and $A \in \alg{Z}$, 
	
\begin{align}
		\tz{\thp}(z_{t}, A)&=\idotsint%\int\ldots\int
		\1{A}(x_{t+1},y_{t+1},\epartilb{t}{1:N},\auxrvb_{t+1}^{1:N})\hiddenstrue(x_t,x_{t+1})\emdenstrue(x_{t+1},y_{t+1})\,\mu(dx_{t+1})\,\refg(dy_{t+1})
		\\&\qquad\times	\int_{\epartb{t}{1:N}}\prod_{k=1}^{N}\propdens{\thp}(\epartilb{t-1}{k},\epartb{t}{k},y_t)\prod_{j=1}^{N}\left(\sum_{i=1}^{N}\frac{\wgtfuncb{\thp}{}(\epartilb{t-1}{i},\epartb{t}{i},y_t)}{\sum_{\ell=1}^{N}\wgtfuncb{\thp}{}(\epartilb{t-1}{\ell},\epartb{t}{\ell},y_t)} \delta_{\epartb{t}{i}}(d\epartilb{t}{j})\right)\foraistats{\,}\forarxiv{\\&\hspace{8cm}}\times\mu^{\tensprod N}(d\epartb{t}{1:N})\,\indmeas^{\tensprod N}(d\auxrvb_{t+1}^{1:N})
		%\\&\ge\int\cdots\int\1{A}(x_{t+1},y_{t+1},\epartilb{t}{1:N},\auxrvb_{t+1}^{1:N})\eq\emdenstrue(x_{t+1},y_{t+1})\,\mu(dx_{t+1})\,\refg(dy_{t+1})
		%\\&\qquad\times	\int_{\epartb{t}{1:N}\in\set{X}^N}\prod_{k=1}^{N}\eq	\prod_{j=1}^{N}\left(\frac{\eq^{6}}{N}\sum_{i=1}^{N}\delta_{\epartb{t}{i}}(d\epartilb{t}{j})\right)\,\mu^{\tensprod N}(d\epartb{t}{1:N})\,\indmeas^{\tensprod N}(d\auxrvb_{t+1}^{1:N})
		\\&\geq\eq^{7N+1}\int\cdots\int\1{A}(x_{t+1},y_{t+1},\epartilb{t}{1:N},\auxrvb_{t+1}^{1:N})\emdenstrue(x_{t+1},y_{t+1})\,\mu(dx_{t+1})\,\refg(dy_{t+1})
		\\&\qquad\times \int_{\epartb{t}{1:N}}%\in\set{X}^N}
		\prod_{j=1}^{N}\left(\frac{1}{N}\sum_{i=1}^{N}\delta_{\epartb{t}{i}}(d\epartilb{t}{j})\right)\,\mu^{\tensprod N}(d\epartb{t}{1:N})\,\indmeas^{\tensprod N}(d\auxrvb_{t+1}^{1:N}).
	\end{align}
	Thus, we may conclude that  $T_\thp$ allows  $\set{Z}$ as a $\nu_1$-small set with 
	\begin{equation}\label{eq:nu_1}
	\nu_1(dz)=\eq^{7N+1}\emdenstrue(x,y)\,\mu(dx)\,\refg(dy)\,\int_{\epartb{}{1:N}}%\in\set{X}^N}
	\prod_{j=1}^{N}\left(\frac{1}{N}\sum_{i=1}^{N}\delta_{\epartb{}{i}}(d\epartilb{}{j})\right)\,\mu^{\tensprod N}(d\epartb{}{1:N})\,\indmeas^{\tensprod N}(d\auxrvb^{1:N}).
	\end{equation}
	Then, by \citet[Theorem 16.0.2(\emph{v})]{meyn:tweedie:2009} or \citet[Theorem~15.3.1(\emph{iii})]{douc:moulines:priouret:soulier:2018} it follows that $(\z_t^\thp)_{t\ge 0}$ is geometrically uniformly ergodic. Then the Dobrushin coefficient of $\tz{\thp}$ is strictly less than one for all $\thp\in\parspace$, which implies that $\tz{\thp}$ admits a unique stationary distribution $\statmeas{\thp}$ \citep[see, \eg,][Theorems~18.2.4--5]{douc:moulines:priouret:soulier:2018}. 
\end{proof}

\subsubsection{Stochastic approximation update and mean field}

Now, define, for $z = (x, y, \tilde{x}^{1:N}, v^{1:N}) \in \set{Z}$,  
\begin{equation}
	\grad{\thp}(z)\coloneqq \nabla\ln \left(\sum_{i=1}^{N}\wgtfuncb{\thp}{}(\epartilb{}{i},\repfunc{\thp}(\epartilb{}{i},y,\auxrvb_{}^i),y)\right)=\frac{\sum_{i=1}^{N}\nabla\wgtfuncb{\thp}{}(\epartilb{}{i},\repfunc{\thp}(\epartilb{}{i},y,\auxrvb_{}^i),y)}{\sum_{i=1}^{N}\wgtfuncb{\thp}{}(\epartilb{}{i},\repfunc{\thp}(\epartilb{}{i},y,\auxrvb_{}^i),y)},
\end{equation}
where, here and everywhere in the following, gradients are with respect to $\thp$, \ie, $\nabla = \nabla_\thp$.  
Then Algorithm~\ref{algo:ovpf_theory} is equivalent with the Robbins--Monro \citep{robbins:monro:1951} stochastic-approximation scheme 
$$
\thp_{t+1}\gets\thp_t+\gamma_{t+1}\grad{\thp_t}(\z_{t+1}), \quad t\in\nset,
$$
where $\z_{t+1}\sim \tz{\thp_t}(Z_t,\cdot)$, initialized by some starting guess $\thp_{0}$. This procedure aims to find a zero of the mean field
$$
\mf{\thp}\eqdef\int \grad{\thp}(z)\,\statmeas{\thp}(dz).
$$
The next result, Proposition~\ref{prop:poisson}, establishes the boundedness of the mean field $\mf{\thp}$, the convergence of the expectation of $\grad{\thp}(\z_t^\thp)$ to the same  and the existence of a solution to the Poisson equation associated with $\tz{\thp}$. This intermediate result will be instrumental in the proofs of Section~\ref{subsec:main_proofs}. Proposition~\ref{prop:poisson} will be proven under the assumption that the gradients of the model and proposal densities are bounded and Lipschitz in $\thp$.

\begin{assumption}\label{assum:3}
	There exists $\kq\in[1,\infty)$ such that for all $(\thp,\thp') \in \parspace^2$, $(x,x')\in\set{X}^2$, $y\in\set{Y}$ and $\auxrvb\in\set{E}$, 
	\begin{align}
		&\max \{\norm{\nabla \hiddens{\thp}(x,\repfunc{\thp}(x,y,\auxrvb))},\norm{\nabla \emdens{\thp}(\repfunc{\thp}(x,y,\auxrvb),y)},\norm{\nabla \propdens{\thp}(x,\repfunc{\thp}(x,y,\auxrvb),y)}\}\le \kq,
		\\& \max \{\abs{ \hiddens{\thp}(x,x')- \hiddens{\thp'}(x,x')},\abs{ \hiddens{\thp}(x,\repfunc{\thp}(x,y,\auxrvb))- \hiddens{\thp'}(x,\repfunc{\thp'}(x,y,\auxrvb))}, 
		\\&\hspace{3cm}\norm{\nabla \hiddens{\thp}(x,\repfunc{\thp}(x,y,\auxrvb))-\nabla \hiddens{\thp'}(x,\repfunc{\thp'}(x,y,\auxrvb))}\}\le \kq \norm{\thp-\thp'},
		\\& \max \{\abs{\emdens{\thp}(x,y)- \emdens{\thp'}(x,y)}, \abs{\emdens{\thp}(\repfunc{\thp}(x,y,\auxrvb),y)- \emdens{\thp'}(\repfunc{\thp'}(x,y,\auxrvb),y)}, \\&\hspace{3cm}\norm{\nabla \emdens{\thp}(\repfunc{\thp}(x,y,\auxrvb),y)-\nabla \emdens{\thp'}(\repfunc{\thp'}(x,y,\auxrvb),y)}\}\le \kq \norm{\thp-\thp'},
		\\&\max\{\abs{\propdens{\thp}(x,x',y)-\propdens{\thp'}(x,x',y)},\abs{\propdens{\thp}(x,\repfunc{\thp}(x,y,\auxrvb),y)-\propdens{\thp'}(x,\repfunc{\thp'}(x,y,\auxrvb),y)},
		\\&\hspace{3cm}\norm{\nabla\propdens{\thp}(x,\repfunc{\thp}(x,y,\auxrvb),y)-\nabla\propdens{\thp'}(x,\repfunc{\thp'}(x,y,\auxrvb),y)}\}\le \kq \norm{\thp-\thp'}. 
	\end{align}
	
\end{assumption}

\begin{proposition}\label{prop:poisson}
	Under Assumptions \ref{assum:2} and \ref{assum:3} the following holds true.
	\begin{enumerate}[label=(\roman*)]
		\item $\mf{\theta}$ is well-defined and bounded on $\parspace$.\label{lemmapart:1}
		\item For every $\thp\in\parspace$, $\mf{\thp}=\lim_{t\to \infty}\mathbb{E}_\thp[\grad{\thp}(\z_t^\thp)]$. \label{lemmapart:2}
		\item \label{lemmapart:3}There exists a measurable function $\htil{\thp}$ on $\set{Z}$ that satisfies the Poisson equation 
		\begin{equation}
			\grad{\thp}(z)-\mf{\thp}=\htil{\thp}(z)-\tz{\thp}\htil{\thp}(z),
		\end{equation}
		for every $\thp\in \parspace$ and $z\in\set{Z}$.
		\item There exists a real number $\aq\in[1,\infty)$ such that for every $(\thp,\thp') \in \parspace^2$ and $z\in\set{Z}$, 
		\begin{align}
			&\max \{\norm{\grad{\thp}(z)},\norm{\htil{\thp}(z)}, 	\norm{\tz{\thp}\htil{\thp}(z)} \}\le \aq, 
			\\&\max \{\norm{\grad{\thp}(z)-\grad{\thp'}(z)},\norm{\tz{\thp}\htil{\thp}(z)-\tz{\thp'}\htil{\thp'}(z)}, \norm{\mf{\thp}-\mf{\thp'}}\}\le 	\aq\norm{\thp-\thp'}. 
		\end{align}
		\label{lemmapart:4}
	\end{enumerate}
\end{proposition}

The proof of Proposition~\ref{prop:poisson} is found in Section~\ref{subsec:otherproofs}.

\subsection{Proof of the main results}\label{subsec:main_proofs}
We are now ready to prove Theorem~\ref{thm:main} and Corollary~\ref{corollary:decay}, which are restated in some more detail below. Our proofs will be based on theory developed by \citet{karimi:2019}, where the minimization of a non-convex smooth objective function is considered. Thus, we assume, in Assumption~\ref{assum:lyapunov}, that the mean field $\mf{\thp}$ is indeed the gradient of some smooth function of $\thp$, the so-called \emph{Lyapunov} function (depending on $N$), maximized by the algorithm. However, since the `surrogate' gradient considered in \texttt{VSMC}, whose time-normalized asymptotic limit is the target of \texttt{OVSMC}, is a biased approximation of the gradient of the ELBO $\mathcal{L}^{\texttt{SMC}}$ (see Section~\ref{subsec:varinf}), the Lyapunov function does not have a straightforward interpretation in this case. 
\begin{assumption}\label{assum:lyapunov}
	There exists a bounded function $V$ on $\Theta$ (the \emph{Lyapunov} function) such that $h=\nabla V$.
\end{assumption}
\begin{assumption}\label{assum:stepsize}
	For every $t\in\nsetpos$, $\gamma_{t+1}\le\gamma_{t}$. In addition, there exist constants $a > 0$ and $a' > 0$ such that for every $t$,
	\begin{equation}
		\gamma_{t}\le a\gamma_{t+1},\quad\gamma_{t}-\gamma_{t+1}\le a' \gamma_{t+1}^2,\quad\gamma_{1}\le 1/(2\aq+2\ch),
	\end{equation}
	where the constant $\aq\in[1,\infty)$ is given in Proposition~\ref{prop:poisson} and $\ch\eqdef\aq(a+1)/2 +\aq^2(2a +1)+\aq a'$.
\end{assumption}
\begin{theorem}[Theorem~\ref{thm:main}]\label{thm:karimi}
	Let Assumptions~\ref{assum:ssmapp}, \ref{assum:2}, \ref{assum:3}, \ref{assum:lyapunov} and \ref{assum:stepsize} hold. Then for every $t \in \nset$, 
	\begin{equation}
		\E[\norm{\mf{\thp_\tau}}^2]\le \frac{2(d_{0,t}+c_{0,t}+(4\aq^3+c_\gamma)\sum_{s=0}^t\gamma_{s+1}^2)}{\sum_{s=0}^t\gamma_{s+1}},
	\end{equation}
	where $\tau\sim\catdist((\gamma_{s+1})_{s=0}^t)$, $\aq$ is defined in Proposition~\ref{prop:poisson} and
	$$
	d_{0,t} \eqdef \E[V(\thp_{t+1})-V(\thp_0)], \quad c_\gamma \eqdef \aq^2(2+\aq), \quad c_{0,t} \eqdef \aq (\gamma_1-\gamma_{t+1}+2).
	$$
\end{theorem}

\begin{proof}
	The proof follows directly from \citet[Theorem 2]{karimi:2019}, with $V$, $\grad{\thp}$ and $h$ being multiplied by $-1$ as we deal with a maximization problem. We notice that \citep[Assumptions~A1--A3]{karimi:2019} are satisfied by our Assumption~\ref{assum:lyapunov}, with  $c_0=d_0=0$ and $c_1=d_1=1$, and by Proposition~\ref{prop:poisson}, letting $L=\aq$. Moreover, \citep[Assumptions~A5--A7]{karimi:2019} are satisfied by Proposition~\ref{prop:poisson}, letting $L_{PH}^{(0)}=L_{PH}^{(1)}=\sigma=\aq$. 
\end{proof}

\begin{corollary}[Corollary~\ref{corollary:decay}]
	Let the assumptions of Theorem~\ref{thm:karimi} hold and let the step-size sequence $(\gamma_t)_{t \in \nset}$ be given by $\gamma_{t}=1/(\sqrt{t}(2\aq+2\ch))$, where $\aq$ and $\ch$ are provided in Assumption~\ref{assum:stepsize}. Then for every $t\in\nsetpos$,
	\begin{equation}
		\E[\norm{\mf{\thp_\tau}}^2] = \mathcal{O}\left(\frac{\log t}{\sqrt{t}}\right), 
	\end{equation}
	where $\tau\sim\catdist((\gamma_{s+1})_{s=0}^t)$. 
\end{corollary}
\begin{proof}
	Noticing that Assumption~\ref{assum:stepsize} is satisfied with $a=2$ and $a'=2(\aq+\ch)$, the result is a direct implication of Theorem~\ref{thm:karimi}.
\end{proof}

\subsection{Auxiliary results}\label{subsec:otherproofs}

In this section we establish Proposition~\ref{prop:poisson}. The proof is prefaced by two lemmas. 

\begin{lemma}\label{lemma:boundswgt}
	Let Assumptions~\ref{assum:2} and \ref{assum:3} hold. Then there exists $\kqt\in[1,\infty)$ such that for all $(\thp,\thp')\in \parspace^2$, $(x,x')\in\set{X}^2$, $y\in\set{Y}$ and $\auxrvb\in\set{E}$, 
	$$
	\norm{\nabla\wgtfuncb{\thp}{}(x,\repfunc{\thp}(x,y,\auxrvb),y)}\le\kqt
	$$
	and 
		\begin{multline*}
		\max\{\abs{\wgtfuncb{\thp}{}(x,x',y)-\wgtfuncb{\thp'}{}(x,x',y)},\abs{\wgtfuncb{\thp}{}(x,\repfunc{\thp}(x,y,\auxrvb),y)-\wgtfuncb{\thp'}{}(x,\repfunc{\thp'}(x,y,\auxrvb),y)},
		\\ \norm{\nabla\wgtfuncb{\thp}{}(x,\repfunc{\thp}(x,y,\auxrvb),y)-\nabla\wgtfuncb{\thp'}{}(x,\repfunc{\thp'}(x,y,\auxrvb),y)}\}\le \kqt \norm{\thp-\thp'}. 
	\end{multline*}
	\end{lemma}
\begin{proof}
First, write 
\begin{multline}
	\nabla\wgtfuncb{\thp}{}(x,\repfunc{\thp}(x,y,\auxrvb),y)=\frac{1}{\propdens{\thp}(x,\repfunc{\thp}(x,y,\auxrvb),y)}\Big(\emdens{\thp}(\repfunc{\thp}(x,y,\auxrvb),y)\nabla\hiddens{\thp}(x,\repfunc{\thp}(x,y,\auxrvb))
	\\+\hiddens{\thp}(x,\repfunc{\thp}(x,y,\auxrvb))\nabla\emdens{\thp}(\repfunc{\thp}(x,y,\auxrvb),y)
	\\-\wgtfuncb{\thp}{}(x,\repfunc{\thp}(x,y,\auxrvb),y)\nabla\propdens{\thp}(x,\repfunc{\thp}(x,y,\auxrvb),y)\Big), 
\end{multline}
implying, from Assumptions~\ref{assum:2}~and~\ref{assum:3},
\begin{multline}
	\norm{\nabla\wgtfuncb{\thp}{}(x,\repfunc{\thp}(x,y,\auxrvb),y)}\le\frac{1}{\eq^2}(\norm{\nabla\hiddens{\thp}(x,\repfunc{\thp}(x,y,\auxrvb))}+\norm{\nabla\emdens{\thp}(\repfunc{\thp}(x,y,\auxrvb),y)})+\frac{1}{\eq^4}\norm{\nabla\propdens{\thp}(x,\repfunc{\thp}(x,y,\auxrvb),y)}
	\\=\frac{\kq(2\eq^2+1)}{\eq^4}. 
\end{multline}
In order to show that $\wgtfuncb{\thp}{}$ is Lipschitz, we apply Assumptions~\ref{assum:2}~and~\ref{assum:3} according to 
\begin{multline}
	\abs{\wgtfuncb{\thp}{}(x,x',y)-\wgtfuncb{\thp'}{}(x,x',y)}\le \frac{\hiddens{\thp}(x,x')\abs{\emdens{\thp}(x',y)-\emdens{\thp'}(x',y)}+\emdens{\thp'}(x',y)\abs{\hiddens{\thp}(x,x')-\hiddens{\thp'}(x,x')}}{\propdens{\thp}(x,x',y)}
	\\+\hiddens{\thp'}(x,x')\emdens{\thp'}(x',y)\frac{\abs{\propdens{\thp'}(x,x',y)-\propdens{\thp}(x,x',y)}}{\propdens{\thp}(x,x',y)\propdens{\thp'}(x,x',y)}\le\frac{\kq(2\eq^2+1)}{\eq^4}\norm{\thp-\thp'}.
\end{multline}
Proceeding similarly, we obtain
\begin{equation}\label{eq:lipsch_w_rep}
	\abs{\wgtfuncb{\thp}{}(x,\repfunc{\thp}(x,y,\auxrvb),y)-\wgtfuncb{\thp'}{}(x,\repfunc{\thp'}(x,y,\auxrvb),y)}\le \frac{(2\eq^2+1)\kq}{\eq^4}\norm{\thp-\thp'}.
\end{equation}
Moreover, in order to show that also the gradient $\nabla \wgtfuncb{\thp}{}$ is Lipschitz, consider the decomposition 
\begin{align} 
	\lefteqn{\norm{\nabla\wgtfuncb{\thp}{}(x,\repfunc{\thp}(x,y,\auxrvb),y)-\nabla\wgtfuncb{\thp'}{}(x,\repfunc{\thp'}(x,y,\auxrvb),y)}} \hspace{20mm}
	\\ 
	\le& \left\lVert\frac{\emdens{\thp}(\repfunc{\thp}(x,y,\auxrvb),y)\nabla\hiddens{\thp}(x,\repfunc{\thp}(x,y,\auxrvb))}{\propdens{\thp}(x,\repfunc{\thp}(x,y,\auxrvb),y)}-\frac{\emdens{\thp'}(\repfunc{\thp'}(x,y,\auxrvb),y)\nabla\hiddens{\thp'}(x,\repfunc{\thp'}(x,y,\auxrvb))}{\propdens{\thp'}(x,\repfunc{\thp'}(x,y,\auxrvb),y)}\right\rVert
	\\
	&+\left\lVert\frac{\hiddens{\thp}(x,\repfunc{\thp}(x,y,\auxrvb))\nabla\emdens{\thp}(\repfunc{\thp}(x,y,\auxrvb),y)}{\propdens{\thp}(x,\repfunc{\thp}(x,y,\auxrvb),y)}-\frac{\hiddens{\thp'}(x,\repfunc{\thp'}(x,y,\auxrvb))\nabla\emdens{\thp'}(\repfunc{\thp'}(x,y,\auxrvb),y)}{\propdens{\thp'}(x,\repfunc{\thp'}(x,y,\auxrvb),y)}\right\rVert
	\\
	&+\left\lVert\frac{\wgtfuncb{\thp'}{}(x,\repfunc{\thp'}(x,y,\auxrvb),y)\nabla\propdens{\thp'}(x,\repfunc{\thp'}(x,y,\auxrvb),y)}{\propdens{\thp'}(x,\repfunc{\thp'}(x,y,\auxrvb),y)}-\frac{\wgtfuncb{\thp}{}(x,\repfunc{\thp}(x,y,\auxrvb),y)\nabla\propdens{\thp}(x,\repfunc{\thp}(x,y,\auxrvb),y)}{\propdens{\thp}(x,\repfunc{\thp}(x,y,\auxrvb),y)}\right\rVert, \label{eq:lipsch_grad_decomp}
\end{align}
where, by Assumptions~\ref{assum:2}~and~\ref{assum:3},
\begin{align}
	\lefteqn{\left\lVert\frac{\emdens{\thp}(\repfunc{\thp}(x,y,\auxrvb),y)\nabla\hiddens{\thp}(x,\repfunc{\thp}(x,y,\auxrvb))}{\propdens{\thp}(x,\repfunc{\thp}(x,y,\auxrvb),y)}-\frac{\emdens{\thp'}(\repfunc{\thp'}(x,y,\auxrvb),y)\nabla\hiddens{\thp'}(x,\repfunc{\thp'}(x,y,\auxrvb))}{\propdens{\thp'}(x,\repfunc{\thp'}(x,y,\auxrvb),y)}\right\rVert} \hspace{20mm}
	\\ 
	\le & 
	\frac{\emdens{\thp}(\repfunc{\thp}(x,y,\auxrvb),y)\norm{\nabla\hiddens{\thp}(x,\repfunc{\thp}(x,y,\auxrvb))-\nabla\hiddens{\thp'}(x,\repfunc{\thp'}(x,y,\auxrvb))}}{\propdens{\thp}(x,\repfunc{\thp}(x,y,\auxrvb),y)}
	\\
	&+ \frac{\norm{\nabla\hiddens{\thp'}(x,\repfunc{\thp'}(x,y,\auxrvb))}\abs{\emdens{\thp}(\repfunc{\thp}(x,y,\auxrvb),y)-\emdens{\thp'}(\repfunc{\thp'}(x,y,\auxrvb),y)}}{\propdens{\thp}(x,\repfunc{\thp}(x,y,\auxrvb),y)}
	\\
	&+ \norm{\nabla\hiddens{\thp'}(x,\repfunc{\thp'}(x,y,\auxrvb))}\emdens{\thp'}(\repfunc{\thp'}(x,y,\auxrvb),y)\frac{\abs{\propdens{\thp'}(x,\repfunc{\thp'}(x,y,\auxrvb),y)-\propdens{\thp}(x,\repfunc{\thp}(x,y,\auxrvb),y)}}{\propdens{\thp}(x,\repfunc{\thp}(x,y,\auxrvb),y)\propdens{\thp'}(x,\repfunc{\thp'}(x,y,\auxrvb),y)}
	\\
	\le &\frac{\kq}{\eq^3}(\eq + \eq^2\kq +\kq)\norm{\thp-\thp'}.
\end{align}
Using \eqref{eq:lipsch_w_rep}, the other terms of \eqref{eq:lipsch_grad_decomp} may be treated similarly, yielding
\begin{align}
	\lefteqn{\norm{\nabla\wgtfuncb{\thp}{}(x,\repfunc{\thp}(x,y,\auxrvb),y)-\nabla\wgtfuncb{\thp'}{}(x,\repfunc{\thp'}(x,y,\auxrvb),y)}} \hspace{15mm}
	\\ 
	&\le \left(\frac{2\kq}{\eq^3}(\eq + \eq^2\kq +\kq)+\frac{\kq}{\eq^5}(\eq+2\eq^2\kq+2\kq)\right)\norm{\thp-\thp'}
	\\
	&= (\eq+2\eq^3+2\kq+2\eq^4\kq +4\eq^2\kq)\frac{\kq}{\eq^5}\norm{\thp-\thp'}.
\end{align}
The proof is the concluded by letting
\begin{equation}
	\kqt \eqdef \max\left\{\frac{(2\eq^2+1)\kq}{\eq^4},(\eq+2\eq^3+2\kq+2\eq^4\kq +4\eq^2\kq)\frac{\kq}{\eq^5}\right\}=(\eq+2\eq^3+2\kq+2\eq^4\kq +4\eq^2\kq)\frac{\kq}{\eq^5}.
\end{equation}
\end{proof}

Our second prefatory lemma establishes Lipschitz continuity and exponential contraction of the Markov dynamics underlying the state-dependent process $(Z_t)_{t \in \nset}$. Recall that under Assumptions~\ref{assum:ssmapp} and \ref{assum:2}, Proposition~\ref{prop:ergod} provides the existence of a unique stationary distribution $\statmeas{\thp}$ of the canonical Markov chain $(\z_t^\thp)_{t \in \nset}$ induced by $\tz{\thp}$. We may then define, for $t\in\nsetpos$,  
$$
\tzc{\thp}^t:\set{Z}\times\alg{Z}\ni (z,A)\mapsto \tz{\thp}^t(z,A)-\statmeas{\thp}(A),
$$
where $\tz{\thp}^t$ is the $t$-skeleton defined recursively as $\tz{\thp}^1=\tz{\thp}$ and $\tz{\thp}^{s+1}(z,A)=\int \tz{\thp}^s(z,dz') \,\tz{\thp}(z',A)$ for $(z,A)\in\set{Z}\times \alg{Z}$. By convention, we let $\tz{\thp}^0(z,A)\eqdef\delta_z(A)$. %The following lemma
\begin{lemma}\label{lemma:1}
	Let Assumptions~\ref{assum:2} and \ref{assum:3} hold. Then there exists a constant %$\rhoq\in(0,1)$ and 
	$\cq\in[1,\infty)$ (possibly depending on $N$) such that for every $(\thp,\thp') \in \parspace^2$, $z\in\set{Z}$, bounded measurable function $h$ on $\set{Z}$ and $t\in\nset$,
	\begin{itemize}
	\item[(i)] $|\tzc{\thp}^th(z)|\le \rhoq^t \|h\|_\infty $, 
	\item[(ii)] $|\tzc{\thp}^t h(z)-\tzc{\thp'}^t h(z)|\le\cq\rhoq^{t / 2}  \|h\|_\infty \norm{\thp-\thp'} $, 
	\item[(iii)] $\max\{\abs{\statmeas{\thp} h -\statmeas{\thp'}h}, \abs{\tz{\thp}h(z)-\tz{\thp'}h(z)}\}\le\cq \|h\|_\infty \norm{\thp-\thp'}$, 
	\end{itemize}
where $\varrho = 1 - \eq^{7N+1}$. 
\end{lemma}
\begin{proof}
	The first bound (i) follows by Proposition~\ref{prop:ergod}, establishing that $\tz{\thp}$ allows $\set{Z}$ as a $\nu_1$-small set, with $\nu_1$ being defined in \eqref{eq:nu_1}. Then by \citet[Theorem 16.2.4]{meyn:tweedie:2009} it holds that for all $z\in\set{Z}$ and bounded measurable functions $h$ on $\set{Z}$, 
		\begin{equation}\label{eq:contraction_T}
		|\tzc{\thp}^t h (z)| = |\tz{\thp}^t h(z)-\statmeas{\thp}h |\le \rhoq^t \| h \|_\infty ,
	\end{equation}
	where $\rhoq\eqdef 1 - \nu_1(\set{Z}) = 1 - \eq^{7N+1}$. %\noteJO{= \ldots} \noteAM{I don't remember the criticism here.}
	
	We turn to (ii) and (iii) and introduce the short-hand notations 
	\begin{align}
		a_{\thp}^k&\eqdef\propdens{\thp}(\epartilb{t-1}{k},\epartb{t}{k},y_t),
		\\
		\beta_{\thp}^N&\eqdef \left( \sum_{i=1}^{N} \frac{\wgtfuncb{\thp}{}(\epartilb{t-1}{i},\epartb{t}{i},y_t)}{\sum_{\ell=1}^{N}\wgtfuncb{\thp}{}(\epartilb{t-1}{\ell},\epartb{t}{\ell},y_t)} \delta_{\epartb{t}{i}} \right)^{\tensprod N}, 
	\end{align}
	which depend implicitly on $\epartilb{t-1}{1:N}$ and $\epartb{t}{1:N}$. We may then write, for given $z_t\in\set{Z}$ and bounded measurable function $h$ on $\set{Z}$,
	\begin{multline}\label{eq:t_ker_lip}
		\abs{\tz{\thp}h(z_t)-\tz{\thp'}h(z_t)}\le \idotsint %\1{A}
		h(x_{t+1},y_{t+1},\epartilb{t}{1:N},\auxrvb_{t+1}^{1:N})\hiddenstrue(x_t,x_{t+1})\emdenstrue(x_{t+1},y_{t+1})\,\mu(dx_{t+1})\,\refg(dy_{t+1})
		\\\times	\int_{\epartb{t}{1:N} %\in\set{X}^N
		}
		\abs{\beta_{\thp}^N \prod_{k=1}^{N} a_\thp^k -  \beta_{\thp'}^N \prod_{k=1}^{N}a_{\thp'}^k }(d\epartilb{t}{1:N})\,\mu^{\tensprod N}(d\epartb{t}{1:N})\,\indmeas^{\tensprod N}(d\auxrvb_{t+1}^{1:N}). 
	\end{multline}
	Here the total variation measure inside the integral can be bounded according to 
	\begin{align}
		\abs{\beta_{\thp}^N \prod_{k=1}^{N} a_\thp^k -  \beta_{\thp'}^N \prod_{k=1}^{N}a_{\thp'}^k }
		&\le \abs{\beta_{\thp}^N-\beta_{\thp'}^N}%(d\epartilb{t}{1:N})
		\prod_{k=1}^{N}a_\thp^k+\abs{\prod_{k=1}^{N}a_\thp^k-\prod_{k=1}^{N}a_{\thp'}^k}\beta_{\thp'}^N
		\\&\le \frac{1}{\eq^N}\abs{\beta_{\thp}^N-\beta_{\thp'}^N}%(d\epartilb{t}{1:N})
		+\left(\sum_{i'=1}^{N}\big|a_\thp^{i'}-a_{\thp'}^{i'}\big|\prod_{k=1}^{i'-1}a_\thp^k\prod_{k=i'+1}^{N}a_{\thp'}^k\right)\beta_{\thp'}^N%(d\epartilb{t}{1:N})
		\\&\le \frac{1}{\eq^N}\abs{\beta_{\thp}^N-\beta_{\thp'}^N}%(d\epartilb{t}{1:N})
		+\frac{1}{\eq^{N-1}} \sum_{k=1}^{N}\abs{a_\thp^{k}-a_{\thp'}^{k}} \beta_{\thp'}^N,%(d\epartilb{t}{1:N}), 
		\label{eq:boundab} 
	\end{align}
	where we have applied Assumption~\ref{assum:2}. To bound the second term in \eqref{eq:boundab}, we first note that by Assumption~\ref{assum:3},  
	\begin{equation}\label{eq:lipsch_a}
		\sum_{k=1}^{N}\abs{a_\thp^k-a_{\thp'}^k}
		\le  N\kq \norm{\thp-\thp'}.
	\end{equation}
	Moreover, by rewriting the measure $\beta_{\thp}^N$ as  
		$$%\begin{multline}
		\beta_{\thp}^N = \left( \sum_{i =1}^{N}\frac{\wgtfuncb{\thp}{}(\epartilb{t-1}{i},\epartb{t}{i},y_t)}{\sum_{\ell=1}^{N}\wgtfuncb{\thp}{}(\epartilb{t-1}{\ell},\epartb{t}{\ell},y_t)} \delta_{\epartb{t}{i}} \right)^{\tensprod N}
		=\sum_{i_{1:N}\in\{1,\ldots,N\}^N}%^{N}
		\bar{w}_{\thp}^{i_{1:N}}\delta_{\epartb{t}{i_{1:N}}},%(d\epartilb{t}{1:N}),
	$$
	where we have defined $\bar{w}_{\thp}^{i_{1:N}}\eqdef\prod_{j=1}^N (\wgtfuncb{\thp}{}(\epartilb{t-1}{i_j},\epartb{t}{i_j},y_t)/\sum_{\ell=1}^{N}\wgtfuncb{\thp}{}(\epartilb{t-1}{\ell},\epartb{t}{\ell},y_t))$, we may bound the same as $\beta_{\thp}^N \leq \mu(\{ x_t^1, \ldots, x_t^N \}^N) / \eq^{6 N}$, where we have defined the occupation measure 
	$$
	\mu(\{ x_t^1, \ldots, x_t^N \}^N) \eqdef \frac{1}{N^N} \sum_{i_{1:N}\in\{1,\ldots,N\}^N} \delta_{\epartb{t}{i_{1:N}}}. 
	$$
Consequently, by combining this with \eqref{eq:lipsch_a} we obtain the bound 
\begin{equation} \label{eq:bound:second:term}
\frac{1}{\eq^{N-1}} \sum_{k=1}^{N}\abs{a_\thp^{k}-a_{\thp'}^{k}} \beta_{\thp'}^N \leq \frac{\kq N}{\eq^{7N - 1}}\norm{\thp-\thp'} \mu(\{ x_t^1, \ldots, x_t^N \}^N). 
\end{equation}
on the second term in \eqref{eq:boundab}. We now bound the first term in \eqref{eq:boundab}. For this purpose, write 
	\begin{equation} \label{eq:beta:diff:rewrite}
		\abs{\beta_{\thp}^N-\beta_{\thp'}^N}%(d\epartilb{t}{1:N})
		=\sum_{i_{1:N}\in\{1,\ldots,N\}^N} \abs{\bar{w}_{\thp}^{i_{1:N}}-\bar{w}_{\thp'}^{i_{1:N}}}\delta_{\epartb{t}{i_{1:N}}},%(d\epartilb{t}{1:N}),
	\end{equation}
	where, by Lemma~\ref{lemma:boundswgt},
	\begin{align}
		\abs{\bar{w}_{\thp}^{i_{1:N}}-\bar{w}_{\thp'}^{i_{1:N}}}&=\abs{\prod_{j=1}^N \frac{\wgtfuncb{\thp}{}(\epartilb{t-1}{i_j},\epartb{t}{i_j},y_t)}{\sum_{\ell=1}^{N}\wgtfuncb{\thp}{}(\epartilb{t-1}{\ell},\epartb{t}{\ell},y_t)}-\prod_{j=1}^N \frac{\wgtfuncb{\thp'}{}(\epartilb{t-1}{i_j},\epartb{t}{i_j},y_t)}{\sum_{\ell=1}^{N}\wgtfuncb{\thp'}{}(\epartilb{t-1}{\ell},\epartb{t}{\ell},y_t)}} \\
&\le \sum_{j=1}^N\left(\prod_{j'=1}^{j-1} \frac{\wgtfuncb{\thp}{}(\epartilb{t-1}{i_{j'}},\epartb{t}{i_{j'}},y_t)}{\sum_{\ell=1}^{N}\wgtfuncb{\thp}{}(\epartilb{t-1}{\ell},\epartb{t}{\ell},y_t)}\prod_{j'=j+1}^{N} \frac{\wgtfuncb{\thp'}{}(\epartilb{t-1}{i_{j'}},\epartb{t}{i_{j'}},y_t)}{\sum_{\ell=1}^{N}\wgtfuncb{\thp'}{}(\epartilb{t-1}{\ell},\epartb{t}{\ell},y_t)} \right.\\
&\left. \hspace{40mm} \times \abs{\frac{\wgtfuncb{\thp}{}(\epartilb{t-1}{i_j},\epartb{t}{i_j},y_t)}{\sum_{\ell=1}^{N}\wgtfuncb{\thp}{}(\epartilb{t-1}{\ell},\epartb{t}{\ell},y_t)}- \frac{\wgtfuncb{\thp'}{}(\epartilb{t-1}{i_j},\epartb{t}{i_j},y_t)}{\sum_{\ell=1}^{N}\wgtfuncb{\thp'}{}(\epartilb{t-1}{\ell},\epartb{t}{\ell},y_t)}} \vphantom{\prod_{j'=j+1}^{N}} \right) \\
&\le\frac{1}{(\eq^6N)^{N-1}}\sum_{j=1}^{N} \left(\frac{|\wgtfuncb{\thp}{}(\epartilb{t-1}{i_j},\epartb{t}{i_j},y_t)-\wgtfuncb{\thp'}{}(\epartilb{t-1}{i_j},\epartb{t}{i_j},y_t)|}{\sum_{\ell=1}^{N}\wgtfuncb{\thp}{}(\epartilb{t-1}{\ell},\epartb{t}{\ell},y_t)} \right.\\
&\left. \hspace{30mm} +\wgtfuncb{\thp'}{}(\epartilb{t-1}{i_j},\epartb{t}{i_j},y_t)\frac{\sum_{\ell=1}^{N}\abs{\wgtfuncb{\thp'}{}(\epartilb{t-1}{\ell},\epartb{t}{\ell},y_t)-\wgtfuncb{\thp}{}(\epartilb{t-1}{\ell},\epartb{t}{\ell},y_t)}}{\sum_{\ell'=1}^{N}\wgtfuncb{\thp}{}(\epartilb{t-1}{\ell'},\epartb{t}{\ell'},y_t)\sum_{\ell''=1}^{N}\wgtfuncb{\thp'}{}(\epartilb{t-1}{\ell''},\epartb{t}{\ell''},y_t)} \vphantom{\frac{|\wgtfuncb{\thp}{}(\epartilb{t-1}{i_j},\epartb{t}{i_j},y_t)-\wgtfuncb{\thp'}{}(\epartilb{t-1}{i_j},\epartb{t}{i_j},y_t)|}{\sum_{\ell=1}^{N}\wgtfuncb{\thp}{}(\epartilb{t-1}{\ell},\epartb{t}{\ell},y_t)}} \right) \\
&\le\frac{1}{(\eq^6N)^{N-1}} \sum_{j=1}^{N}\left(\frac{\kqt}{N\eq^3}\norm{\thp-\thp'} + \frac{\kqt }{N\eq^{9}}\norm{\thp-\thp'} \right) \\
&=\frac{\kqt N}{N^N\eq^{6N-3}}\left(1+\frac{1}{\eq^6}\right)\norm{\thp-\thp'},  \label{eq:lipsch_wbar}
 	\end{align}
	implying, via, \eqref{eq:beta:diff:rewrite}, that the first term in \eqref{eq:boundab} can be bounded as 
	\begin{equation} \label{eq:bound:first:term}
	 \frac{1}{\eq^N}\abs{\beta_{\thp}^N-\beta_{\thp'}^N} \leq %\frac{1}{\eq^N}
	 \frac{\kqt N}{\eq^{7N-3}}\left(1+\frac{1}{\eq^6}\right)\norm{\thp-\thp'}
	 \mu(\{ x_t^1, \ldots, x_t^N \}^N). 
	\end{equation}
	Thus, combining \eqref{eq:boundab}, \eqref{eq:bound:second:term} and 
 \eqref{eq:bound:first:term} yields 
 \begin{equation} \label{eq:lipsch_wgtbeta}
	\abs{\beta_{\thp}^N \prod_{k=1}^{N} a_\thp^k -  \beta_{\thp'}^N \prod_{k=1}^{N}a_{\thp'}^k } \leq (\eq^6\kqt +\kqt  + \eq^4\kq )\frac{N}{\eq^{7N+3}}\norm{\thp-\thp'}
	 \mu(\{ x_t^1, \ldots, x_t^N \}^N). 
\end{equation}
	Now, by plugging the bound \eqref{eq:lipsch_wgtbeta} into \eqref{eq:t_ker_lip} we obtain
	\begin{multline}\label{eq:lipsch_T}
		\abs{\tz{\thp}h(z_t)-\tz{\thp'}h(z_t)}\le \| h \|_\infty \idotsint\hiddenstrue(x_t,x_{t+1})\emdenstrue(x_{t+1},y_{t+1})\,\mu(dx_{t+1})\,\refg(dy_{t+1})
		\\ \times(\eq^6\kqt +\kqt  + \eq^4\kq )\frac{ N}{\eq^{7N+3}}\norm{\thp-\thp'}\int_{\epartb{t}{1:N}}
		 \mu(\{ x_t^1, \ldots, x_t^N \}^N)
		(d\epartilb{t}{1:N})
		\, \mu^{\tensprod N}(d\epartb{t}{1:N})\,\indmeas^{\tensprod N}(d\auxrvb_{t+1}^{1:N})
		\\
		=   (\eq^6\kqt +\kqt  + \eq^4\kq )\frac{ N}{\eq^{7N+3}} \| h \|_\infty \norm{\thp-\thp'}.
	\end{multline}
	Now, using the decomposition, for $t \in \nsetpos$,  
		$$
		\tz{\thp}^{t+1} -\tz{\thp'}^{t+1} = \sum_{s=0}^{t} \left( \tz{\thp'}^{t-s} \tz{\thp}^{s+1} - \tz{\thp'}^{t-s+1} \tz{\thp}^{s} \right)
		=\sum_{s=0}^{t} \left( \tz{\thp'}^{t-s} \tz{\thp} \tzc{\thp}^{s} - \tz{\thp'}^{t-s} \tz{\thp'} \tzc{\thp}^{s} \right) 
		=\sum_{s=0}^{t}  \tz{\thp'}^{t-s}(\tz{\thp}-\tz{\thp'})\tzc{\thp}^s
	$$
	we obtain, using \eqref{eq:lipsch_T} and (i), the bound 
	\begin{align}
	\abs{\tz{\thp}^{t+1}h(z)-\tz{\thp'}^{t+1}h(z)} &\leq \sum_{s=0}^{t} \int \tz{\thp'}^{t-s}(z,dz') |\tz{\thp} \tzc{\thp}^s h(z') -\tz{\thp'}  \tzc{\thp}^s h(z') | %\abs{\tzc{\thp}^s(z'',A)} 
	\\
	&\leq  (\eq^6\kqt +\kqt  + \eq^4\kq )\frac{ N}{\eq^{7N+3}}\| h \|_\infty\norm{\thp-\thp'}\sum_{s=0}^{t}\rhoq^s \\%\int\tz{\thp'}^{t-s}(z,dz') \\
	 &\leq (\eq^6\kqt +\kqt  + \eq^4\kq )\frac{ N}{\eq^{7N+3}(1-\rhoq)}\| h \|_\infty\norm{\thp-\thp'}.
	\end{align}
Similarly,
	\begin{align}
	  	|\tzc{\thp}^{t+1}h(z)-\tzc{\thp'}^{t+1}h(z)|&=\abs{\sum_{s=0}^{t}\iint \tzc{\thp}^s h(z'')(\tz{\thp}-\tz{\thp'})(z',dz'')\tzc{\thp'}^{t-s}(z,dz')}
		\\
		&\le \sum_{s=0}^{t} \left| \int \tzc{\thp'}^{t-s}(z, dz') (\tz{\thp} \tzc{\thp}^s h(z') -\tz{\thp'} \tzc{\thp}^s h(z')) \right|
		\\
		&\le  (\eq^6\kqt +\kqt  + \eq^4\kq )\frac{ N}{\eq^{7N+3}}\norm{\thp-\thp'}\sum_{s=0}^{t}\rhoq^{t - s} \|\tzc{\thp'}^s h(\cdot) \|_\infty 
		\\
		&\leq  (\eq^6\kqt +\kqt  + \eq^4\kq )\frac{ N}{\eq^{7N+3}}(t+1)\rhoq^t \| h \|_\infty \norm{\thp-\thp'}.
	\end{align}
	Finally, we can write, for arbitrary $t\in\nsetpos$ and $z\in\set{Z}$, 
 	\begin{align}
 		\abs{\statmeas{\thp} h -\statmeas{\thp'} h} &\le |\tz{\thp}^t h(z)-\tz{\thp'}^t h(z)|+|\tzc{\thp}^t h(z)|+|\tzc{\thp'}^t h(z)| \\
		&\le(\eq^6\kqt +\kqt  + \eq^4\kq )\frac{ N}{\eq^{7N+3}(1-\rhoq)} \| h \|_\infty  \norm{\thp-\thp'}+2\rhoq^t \| h \|_\infty, \label{eq:lipscmeas}
 	\end{align}
	implying that 
 	%and since it has to hold for every $t\in\nsetpos$, we have
 	\begin{equation}
 		\abs{\statmeas{\thp} h -\statmeas{\thp'} h}\le (\eq^6\kqt +\kqt  + \eq^4\kq )\frac{ N}{\eq^{7N+3}(1-\rhoq)}   \| h \|_\infty \norm{\thp-\thp'}.
 	\end{equation}
 	The proof of (ii) and (iii) is now concluded by simply noting that 
	$$
	|\tzc{\thp}^{t}h(z)-\tzc{\thp'}^{t}h(z) | \le  (\eq^6\kqt +\kqt  + \eq^4\kq )\frac{ N}{\eq^{7N+3}}t\rhoq^{t/2 -1}\rhoq^{t / 2}  \| h \|_\infty  \norm{\thp-\thp'},
	$$
 	and letting 
 	\begin{equation}
 		\cq\coloneqq (\eq^6\kqt +\kqt  + \eq^4\kq )\frac{ N}{\eq^{7N+3}}\max\left\{ \sup_{t \in \nset} t\rhoq^{t / 2 - 1},\frac{1}{1-\rhoq}\right\}.
 	\end{equation}
\end{proof}

We are now ready to prove Proposition~\ref{prop:poisson}.
\begin{proof}[Proof of Proposition~\ref{prop:poisson}]
	Using Assumptions \ref{assum:2} and \ref{assum:3} and Lemma~\ref{lemma:boundswgt} we conclude that for all $(\thp,\thp') \in \parspace^2$ and $z\in\set{Z}$,  
	\begin{equation}\label{eq:bound_grad}
		\norm{\grad{\thp}(z)}\le \frac{\sum_{i=1}^{N}\norm{\nabla\wgtfuncb{\thp}{}(\epartilb{}{i},\repfunc{\thp}(\epartilb{}{i},y,\auxrvb_{}^i),y)}}{\sum_{\ell=1}^{N}\wgtfuncb{\thp}{}(\epartilb{}{\ell},\repfunc{\thp}(\epartilb{}{\ell},y, \auxrvb_{}^\ell),y)}\le \frac{\kqt}{\eq^3}, 
	\end{equation}
	from which \ref{lemmapart:1}  immediately follows. 
	
	We turn to \ref{lemmapart:2}. Using Lemma~\ref{lemma:1}(i) and \eqref{eq:bound_grad}, for every $t \in \nset$ and $z\in\set{Z}$,
		 \begin{equation}
			\norm{\mathbb{E}_\thp[\grad{\thp}(\z_t^\thp) \mid \z_0^\thp=z]-\mf{\thp}}=\norm{\tz{\thp}^{t}\grad{\thp}(z)-\mf{\thp}}=\norm{\tzc{\thp}^t\grad{\thp}(z)}
			\le \frac{\kqt}{\eq^3}\rhoq^t.
		\end{equation}
	Thus, for every $\thp\in\parspace$,
	\begin{multline*}
	0 \leq  \liminf_{t\to \infty}\norm{\mathbb{E}_\thp[ \grad{\thp}(\z_t^\thp) ] -\mf{\thp}} \leq \limsup_{t\to \infty}\norm{\mathbb{E}_\thp[ \grad{\thp}(\z_t^\thp) ] -\mf{\thp}}\\
	=\limsup_{t\to\infty}\norm{\mathbb{E}_\thp[ \mathbb{E}_\thp[\grad{\thp}(\z_t^\thp)\mid \z_0^\thp]-\mf{\thp}]}\le\limsup_{t\to \infty}\frac{\kqt}{\eq^3}\rhoq^t=0,
	\end{multline*}
	which proves \ref{lemmapart:2}. 
	
	In order to establish \ref{lemmapart:3}, 
	let
	\begin{equation}
		\htil{\thp}(z)\coloneqq\sum_{s=0}^{\infty}(\tz{\thp}^{s}\grad{\thp}(z)-\mf{\thp}), \quad z \in \set{Z}. 
	\end{equation}
	Indeed, again by Lemma~\ref{lemma:1}(i), for every $z \in \set{Z}$,  
	\begin{equation}
		\|\htil{\thp}(z)\| \leq \sum_{s=0}^{\infty}\norm{\tz{\thp}^{s}\grad{\thp}(z)-\mf{\thp}}\le \frac{\kqt}{\eq^3} \sum_{s=0}^{\infty}\rhoq^{s}= \frac{\kqt}{\eq^3(1-\rhoq)}.
	\end{equation}
	which implies that $\htil{\thp}$ and $\tz{\thp}\htil{\thp}$ are well defined and bounded. Moreover, by the dominated convergence theorem, for every $z \in \set{Z}$, 
	\begin{equation}
		\tz{\thp}\htil{\thp}(z)=\sum_{s=1}^{\infty}(\tz{\thp}^{s}\grad{\thp}(z)-\mf{\thp}), 
	\end{equation}
	implying \ref{lemmapart:3}. 
	
	To prove \ref{lemmapart:4}, write, using Lemma~\ref{lemma:boundswgt}, 
	\begin{align}
		\norm{\grad{\thp}(z)-\grad{\thp'}(z)} &\le \frac{\sum_{i=1}^{N}\norm{\nabla\wgtfuncb{\thp}{}(\epartilb{}{i},\repfunc{\thp}(\epartilb{}{i},y,\auxrvb_{}^i),y)-\nabla\wgtfuncb{\thp'}{}(\epartilb{}{i},\repfunc{\thp'}(\epartilb{}{i},y,\auxrvb_{}^i),y)}}{\sum_{\ell=1}^{N}\wgtfuncb{\thp}{}(\epartilb{}{\ell},\repfunc{\thp}(\epartilb{}{\ell},y, \auxrvb_{}^\ell),y)}
		\\
		&\hspace{20mm} +\frac{\norm{\grad{\thp'}(z)}\sum_{i=1}^{N}\abs{\wgtfuncb{\thp}{}(\epartilb{}{i},\repfunc{\thp}(\epartilb{}{i},y,\auxrvb_{}^i),y)-\wgtfuncb{\thp'}{}(\epartilb{}{i},\repfunc{\thp'}(\epartilb{}{i},y,\auxrvb_{}^i),y)}}{\sum_{\ell=1}^{N}\wgtfuncb{\thp}{}(\epartilb{}{\ell},\repfunc{\thp}(\epartilb{}{\ell},y,\auxrvb_{}^\ell),y)} \\
		&\le\left(\frac{\kqt}{\eq^3}+\frac{\kqt^2}{\eq^6}\right)\norm{\thp-\thp'}, \label{eq:lipsch_grad}
	\end{align}
	and by \eqref{eq:lipsch_grad} and Lemma~\ref{lemma:1} it holds that for every $z \in \set{Z}$, 
	\begin{align}
		\norm{\tzc{\thp}^t\grad{\thp}(z)-\tzc{\thp'}^t\grad{\thp'}(z)} &\le  \norm{\tzc{\thp}^t \grad{\thp}(z')- \tzc{\thp}^t \grad{\thp'}(z')} 
		- \norm{\tzc{\thp}^t\grad{\thp'}(z)-\tzc{\thp'}^t\grad{\thp'}(z)} 
		\\
		&\le \left(\frac{\kqt}{\eq^3}+\frac{\kqt^2}{\eq^6}\right) \rhoq^t \norm{\thp-\thp'}+\frac{\kqt}{\eq^3} \cq\rhoq^{t / 2} \norm{\thp-\thp'} \\
		&\le 2 \left(\frac{\kqt}{\eq^3}+\frac{\kqt^2}{\eq^6}\right) \cq\rhoq^{t / 2}\norm{\thp-\thp'}.
	\end{align}
	Thus, we may write
	\begin{align}
	\norm{\tz{\thp}\htil{\thp}(z)-\tz{\thp'}\htil{\thp'}(z)} &= \left \| \sum_{s=1}^{\infty}(\tz{\thp}^{s}\grad{\thp}(z)-\mf{\thp})-\sum_{s=1}^{\infty}(\tz{\thp'}^{s}\grad{\thp'}(z)-\mf{\thp'}) \right\| \\
	&\leq \sum_{s=0}^{\infty}\norm{\tzc{\thp}^{s}\grad{\thp}(z)-\tzc{\thp'}^{s}\grad{\thp'}(z)} \\
	&\le \left(\frac{\kqt}{\eq^3}+\frac{\kqt^2}{\eq^6}\right) \frac{2\cq}{1-\sqrt{\rhoq}}\norm{\thp-\thp'}. 
	\end{align}
	Finally, by Lemma~\ref{lemma:1} and \eqref{eq:lipsch_grad} again, we have
	\begin{equation}
		\norm{\mf{\thp}-\mf{\thp'}} \le\int\norm{ \grad{\thp}(z)-\grad{\thp'}(z)}\, \statmeas{\thp}(dz)+ 
		%\int \norm{ \grad{\thp'}(z)}\abs{\statmeas{\thp}-\statmeas{\thp'}}(dz) 
		\|\statmeas{\thp} \grad{\thp'} -\statmeas{\thp'} \grad{\thp'} \| 
		\le \left(\frac{\kqt}{\eq^3}+\frac{\kqt^2}{\eq^6}\right)(1+\cq)\norm{\thp-\thp'}, 
	\end{equation}
	which allows us, by letting 
	$$\aq\eqdef \left(\frac{\kqt}{\eq^3}+\frac{\kqt^2}{\eq^6}\right) \frac{2\cq}{1-\sqrt{\rhoq}},
	$$
	to conclude the proof of \ref{lemmapart:4}.
\end{proof}

\section{ESS improvement for the model in Section~\ref{subsec:lg}}\label{sec:ess}
Figure~\ref{fig:lg_essall} shows how the effective sample size \citep[ESS,][]{liu:1996} of the particle cloud improves while $\propdens{\thprop}$ is being learned in univariate linear Gaussian model of Section~\ref{subsec:lg}, to finally reach the performance of the optimal proposal. This is evident for $S_v=0.2$; on the other hand, when $S_v=1.2$, the particles are well propagated into regions of non-negligible likelihood even with the bootstrap proposal, resulting in the normalized ESS being close to one regardless. 
\begin{figure}[htb]
	\centering
		\includegraphics[width=\foraistats{0.49}\forarxiv{0.49}\columnwidth]{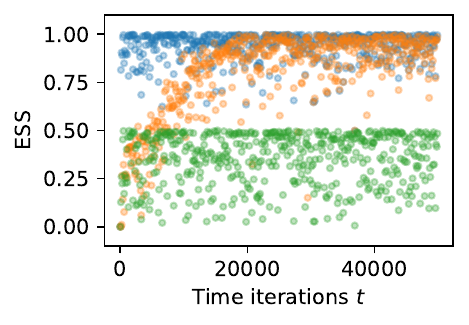}
		\includegraphics[width=\foraistats{0.49}\forarxiv{0.49}\columnwidth]{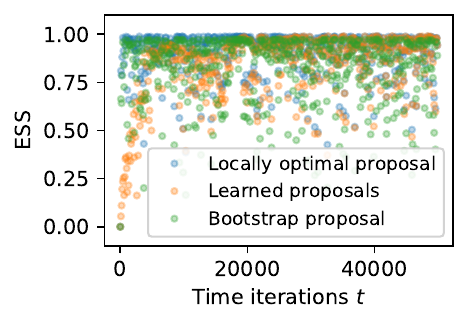}
	\caption{Evolution of (every $100^{\tiny{\mbox{th}}}$) normalized ESS for the one-dimensional linear Gaussian SSM in Section~\ref{subsec:lg},  for $S_v=0.2$ (left) and $S_v=1.2$ (right).}\label{fig:lg_essall}
\end{figure}

\section{Details on the deep generative model of Section~\ref{subsec:ball}}\label{appendix:b}
In this section we provide more details on the model presented in Section~\ref{subsec:ball}.

\subsection{Data generation}\label{subsec:datagenball}
First, we describe how the frames of the video are generated. The movement can shortly be described as a partially observed Gaussian random walk confined to a rectangle. More specifically, the agent moves on the rectangle $[0,1]\times [0,5]\subset\rset^2$, starting from a uniformly sampled point on $[0,1]\times\{5\}$. It then moves according to a bivariate Gaussian random walk with covariance matrix $0.004I$ and vertical drift being initially $-0.15$ and changing sign every time that either bottom or top edges are hit. In practice, the agent bounces every time it hits any edge. Then each frame is created by projecting the rectangle into a $32\times 32$ array and giving the agent an approximately round shape by overlapping, as a plus sign, two $3\times 5$ rectangles, one vertical and one horizontal. In the created arrays, the background has value zero, while the pixels representing the agent are equal to one. All the frames have an horizontal $16\times 30$ rectangle in the center, which (partially) occludes the view of the agent every time it (partially) falls in that area. This area is represented by $0.5$-valued pixels.
\subsection{Model architecture}\label{subsec:vrnnball}
 Inspired by \citet[Section~C.1]{le:2018}, the model is constructed on the basis of the \emph{variational recurrent neural networks} \citep[VRNN, ][]{chung:2015} framework. More precisely, the generative model is represented by the process $(X_{t+1}, H_t, Y_{t+1})_{t\in\nsetpos}$, with joint density 
\begin{equation}
	p(x_{1:T}, h_{0:T}, y_{1:T})= p_{0}(h_0)\prod_{t=0}^{T-1}\hiddens{\thp}(x_{t+1}\mid h_t)\emdens{\thp}(y_{t+1}\mid h_t, x_{t+1})p_{\thprop,\thp}(h_{t+1}\mid h_t,x_{t+1},y_{t+1}),
\end{equation}
where $T$ is a fixed time horizon, $y_{1:T}$ represent the frames of the video, $x_{1:T}$ some lower-dimensional latent states and $h_{0:T}$ are the so-called \emph{hidden states} of the \textit{gated recurrent unit} ({\texttt{GRU}}), which is the specific \emph{recurrent neural network} (RNN) used within the whole architecture. The initial and transition distributions of the generative model are 
\begin{align}
	H_0&\sim \Norm (0,I),
	\\X_{t+1}\mid H_{t}=h_{t}&\sim\Norm(\mu_\thp^x(h_{t}), \sigma_\thp^x(h_t)^2),
	\\Y_{t+1}\mid H_t=h_t,X_{t+1}=x_{t+1}&\sim \operatorname{Bernoulli}(\mu_\thp^y(\varphi_\thp^x(x_{t+1}),h_t)),
	\\H_{t+1}\mid H_t=h_t,X_{t+1}=x_{t+1},Y_{t+1}=y_{t+1}&\sim \delta_{\texttt{GRU}_\thprop(h_t,\varphi_\thp^x(x_{t+1}),\varphi_\thprop^y(y_{t+1}))},
\end{align}
for $t\in\nset$, while the proposal $\propdens{\thprop}(x_{t+1}\mid y_{t+1}, h_t)$ is given by 
\begin{equation}
	X_{t+1}\mid Y_{t+1}=y_{t+1}, H_{t}=h_t\sim \Norm(\mu_\lambda^p(\varphi_\thprop^y(y_{t+1}),h_t), \sigma_\lambda^p(\varphi_\thprop^y(y_{t+1}),h_t)^2).
\end{equation}
More in detail, the model comprises the following neural networks.
\begin{itemize}
	\item $\mu_\thp^x$ and $\sigma_\thp^x$ have two dense layers, the second one being the output, with each $128$ nodes, corresponding then to the size of the latent states, whose first layer is shared. The activations of the first layers are ReLU functions, while the activations of the output layer are linear and softplus, respectively.
	\item $\varphi_\thp^x$ is a single dense output layer with 128 nodes and the ReLU activation function.
	\item $\varphi_\thprop^y$ is the encoder of the frames and is represented by a sequential architecture with four convolutional layers, all with $4\times 4$ filters, stride two and padding one, except the last one which has stride one and zero padding; the numbers of filters are, in order, 32, 128, 64 and 32. The activation functions of the convolutions are leaky ReLUs with slope $0.2$ and we add batch normalization layers after each of them (except the last one which simply has linear activation). In the end the output is flattened to obtain a tensor in $\rset^{32}$.
	\item $\mu_\thp^y$ is the decoder, which is modeled by a sequential architecture with transposed convolutions. In particular, the first layer has 128 $4\times 4$ filters with stride one and zero padding. Then we have two layers with 64 and 32  $4\times 4$ filters, respectively, while the last one has a single $4\times 4$ filter; all these have stride two and padding one. We use again leaky ReLUs as activations with slope 0.2 and batch normalization layers. The activation function of the output layer is instead a sigmoid, in order to obtain an output in $(0,1)^{32\times 32}$.
	\item $\texttt{GRU}_\thprop$ is a \emph{gated recurrent unit} RNN, which takes as input the concatenated outputs of $\varphi_\thp^x$ and $\varphi_\thprop^y$ to produce deterministically the next hidden state $h_{t+1}$ of the RNN. Here $\texttt{GRU}_\thprop$ takes $h_t$ as an additional input to model the recurrence, but in practice, during training we need to input a time series of (functions of) latent states and frames. Since we deal with streaming data, it would be infeasible to input the whole history at every iteration, and we hence include only the 40 most recent latent states and frames when learning the $\texttt{GRU}_\thprop$.
	\item $\mu_\lambda^p$ and $ \sigma_\lambda^p$ have three dense layers of size 128 including the output and share the first two layers, with ReLUs activations except the last layers, which has linear and softplus functions, respectively.
\end{itemize}
We note that the subscripts to the neural networks defined above indicate which optimizer---the one for the generative model or the one for the proposal---that is used. Even if the {\texttt{GRU}} is supposed to be part of the generative model, we noticed a learning improvement by considering it part of the proposal, motivated by the fact that it is the first (deterministic) sampling operation in the propagation of the particles. In order to describe more clearly our procedure, we have displayed its pseudocode in Algorithm~\ref{algo:ovpf_vrnn}. In our notation, $(\epart{t-39:t}{x,i})_{i=1}^N$ are the particles representing the latent states of the process, while $(\epart{t-40:t-1}{h,i})_{i=1}^N$ refer to the hidden states of the {\texttt{GRU}}.
Note that for most of the variables involved, direct dependencies on the parameters are omitted for a less cumbersome notation. We remark that for iterations $t< 40$, the starting time of the vectors of particles must be considered one for $x$ and zero for $h$.

\begin{algorithm}[htb]
	\caption{{\OVSMC} for deep generative model of moving agent of Section~\ref{subsec:ball}.}\label{algo:ovpf_vrnn}
	\begin{algorithmic}[1]
		\STATE {\bfseries Input:} $(\epart{t-39:t}{x,i},\epart{t-40:t-1}{h,i},\wgt{t}{i})_{i=1}^N, y_{t-39:t+1}, \thp_t,\thprop_t$
		\FOR{$i \gets 1,\dots, L$}
		\STATE draw $\I{t+1}{i}\sim \catdist((\wgt{t}{\ell})_{\ell=1}^N)$;
		\STATE set $\epart{t-39:t}{h,i}\gets \texttt{GRU}_{\thprop_t}(\texttt{initial\_state}=\epart{t-40}{h,\I{t+1}{i}},\varphi_{\thp_t}^x(\epart{t-39:t}{x,\I{t+1}{i}}), \varphi_{\thprop_t}^y(y_{t-39:t}))$;
		\STATE draw $\auxrv_{t+1}^i\sim \Norm(0,I_{128})$;
		\STATE set  $\epart{t+1}{x,i}\gets \mu_{\thprop_t}^x(\varphi_{\thprop_t}^y(y_{t+1}),\epart{t}{h,i})+\sigma_{\thprop_t}^x(\varphi_{\thprop_t}^y(y_{t+1}),\epart{t}{h,i}) \auxrv_{t+1}^i$;
		\STATE set $\wgtfunc{t+1}{i}(\thprop_t,\thp_t)\gets \dfrac{\hiddens{\thp_t}(\epart{t+1}{x,i}\mid \epart{t}{h,i})\emdens{\thp_{t}}(y_{t+1}\mid \epart{t}{h,i},\epart{t+1}{x,i})}{\propdens{\thprop_t}(\epart{t+1}{x,i}\mid y_{t+1},\epart{t}{h,i})}$;
		\ENDFOR
		\STATE set $\thprop_{t+1}\gets\thprop_t+\gamma_{t+1}^\thprop\nabla_{\thprop}\log\left(\sum_{i=1}^N \wgtfunc{t+1}{i}(\thprop_t,\thp_t)\right)$;
		\FOR{$i \gets 1,\dots,N$}
		\STATE draw $\I{t+1}{i}\sim \catdist((\wgt{t}{\ell})_{\ell=1}^N)$;
		\STATE set $\epart{t-39:t}{h,i}\gets \texttt{GRU}_{\thprop_{t+1}}(\texttt{initial\_state}=\epart{t-40}{h,\I{t+1}{i}},\varphi_{\thp_t}^x(\epart{t-39:t}{x,\I{t+1}{i}}), \varphi_{\thprop_{t+1}}^y(y_{t-39:t}))$;
		\STATE draw $\auxrv_{t+1}^i\sim \Norm(0,I_{128})$;
		\STATE set  $\epart{t+1}{x,i}\gets \mu_{\thprop_{t+1}}^x(\varphi_{\thprop_{t+1}}^y(y_{t+1}),\epart{t}{h,i})+\sigma_{\thprop_{t+1}}^x(\varphi_{\thprop_{t+1}}^y(y_{t+1}),\epart{t}{h,i})\auxrv_{t+1}^i$;
		\STATE set $\wgtfunc{t+1}{i}(\thprop_{t+1},\thp_t)\gets \dfrac{\hiddens{\thp_t}(\epart{t+1}{x,i}\mid \epart{t}{h,i})\emdens{\thp_{t}}(y_{t+1}\mid \epart{t}{h,i},\epart{t+1}{x,i})}{\propdens{\thprop_{t+1}}(\epart{t+1}{x,i}\mid y_{t+1},\epart{t}{h,i})}$;
		\STATE set $\epart{t-38:t+1}{x,i}\gets(\epart{t-38:t}{x,\I{t+1}{i}},\epart{t+1}{x,i})$;
		\ENDFOR
		\STATE set $\thp_{t+1}\gets\thp_t+\gamma_{t+1}^\thp\nabla_\thp\log\left(\sum_{i=1}^N \wgtfunc{t+1}{i}(\thprop_{t+1},\thp_t)\right)$;
		\STATE {\bfseries return} $(\epart{t-38:t+1}{x,i},\epart{t-39:t}{h,i},\wgt{t+1}{i})_{i=1}^N,\thp_{t+1},\thprop_{t+1}$.
	\end{algorithmic}
\end{algorithm}

\end{document}